\documentclass{article}

\usepackage{microtype}
\usepackage{graphicx}
\usepackage{subcaption}
\usepackage{enumitem}
\usepackage{booktabs} %
\usepackage[toc,page,header]{appendix}
\usepackage{babel,makecell,adjustbox,multirow}

\usepackage{listings}
\usepackage[dvipsnames]{xcolor}
\usepackage[normalem]{ulem}

\definecolor{codegray}{rgb}{0.5,0.5,0.5}
\definecolor{backcolour}{rgb}{0.96,0.96,0.96}

\usepackage{hyperref}

\PassOptionsToPackage{numbers, compress}{natbib}

\usepackage{amsmath}
\usepackage{amssymb}
\usepackage{mathtools}
\usepackage{amsthm}

\usepackage{bm}
\usepackage{tikz}
\usetikzlibrary{arrows.meta,positioning,calc,shapes.geometric}

\usepackage[accepted]{icml2026}

\usepackage[capitalize,noabbrev]{cleveref}

\theoremstyle{plain}
\newtheorem{theorem}{Theorem}[section]
\newtheorem{proposition}[theorem]{Proposition}

\theoremstyle{definition}

\theoremstyle{remark}

\newcommand{\bx}{\bm{x}}
\newcommand{\br}{\bm{r}}

\newcommand{\bmu}{\bm{\mu}}

\newcommand{\bv}{\bm{v}}

\newcommand{\valstd}[2]{$#1$ {\tiny $\pm #2$}}
\newcommand{\ourmethod}{ECD}

\long\def\edited#1{#1}
\newcommand{\makeblue}{}

\providecommand{\doparttoc}{}
\providecommand{\faketableofcontents}{}

\usepackage[textsize=tiny]{todonotes}

\icmltitlerunning{Energy-based Compositional Diffusion Planning}

\begin{document}

\doparttoc %
\faketableofcontents %

\twocolumn[
  \icmltitle{Energy-based Compositional Diffusion Planning}

  \icmlsetsymbol{equal}{*}

  \begin{icmlauthorlist}
    \icmlauthor{Tao Sun}{stanford}
    \icmlauthor{Utkarsh A. Mishra}{gatech}
    \icmlauthor{Jiaxin Lu}{austin}
    \icmlauthor{Danfei Xu}{gatech}
    \icmlauthor{Iro Armeni}{stanford}
  \end{icmlauthorlist}

  \icmlaffiliation{stanford}{Stanford University, California, USA}
   \icmlaffiliation{gatech}{Georgia Institute of Technology, Georgia, USA}
    \icmlaffiliation{austin}{University of Texas at Austin, Texas, USA}
  \icmlcorrespondingauthor{Tao Sun}{taosun@stanford.edu}

  \icmlkeywords{Machine Learning, ICML}

  \vskip 0.3in
]

\printAffiliationsAndNotice{}  %

\begin{abstract}
\edited{Compositional diffusion planners aim to solve long-horizon robotic tasks using short training trajectories.}
Yet, current approaches often rely on the heuristic stitching of local predictions. 
\edited{We show that the resulting stitched update is generally a non-conservative field} that does not mathematically
correspond to any valid global trajectory log-density function.
We propose \textbf{Energy-based Compositional Diffuser} (\ourmethod), a framework that formulates the global trajectory as the minimizer of the sum of local bridge potentials. 
\edited{This energy-based perspective defines a conservative correction field and contains a \emph{boundary reaction} term that heuristic stitching omits.}
To enable efficient inference, we further introduce a Markov-based score approximation that computes the reaction term \edited{via} a single block-tridiagonal solve, maintaining time complexity linear in the planning horizon. Empirically, \ourmethod{} achieves state-of-the-art success rates on a range of OGBench stitching tasks, while nearly matching the inference speed of heuristic stitching methods. \edited{Code is available at  \url{https://github.com/GradientSpaces/ECD}.}

\end{abstract}

\section{Introduction}
\label{sec:intro}

Diffusion models have become a robust paradigm for offline decision-making and robotic planning
\citep{janner2022diffuser,ajay2023decisiondiffuser,chi2023diffusionpolicy}. 
By abstracting planning as a sampling process, these models offer distinct advantages over classical planners: they naturally handle multimodality, allow flexible conditioning on goals and constraints, and integrate seamlessly with guidance-style control signals to satisfy multiple objectives at inference time \citep{dhariwal2021diffusionbeatgans,ho2022classifierfree}.

\begin{figure}[t]
    \centering
    \vspace{-0.1em}
\includegraphics[width=1\linewidth]{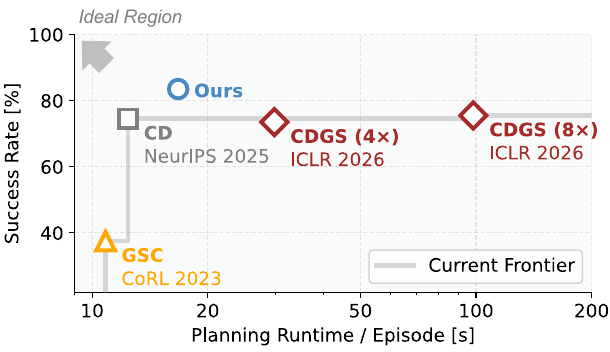}
    \caption{\textbf{Performance-runtime frontier on OGBench stitch tasks} \citep{ogbench_park2024}.
    We plot the trade-off between success rate and inference runtime averaged over PointMaze Giant and AntMaze Giant environments. While recent inference-time search methods like CDGS (4 or 8 resampling rounds) achieve high success, they suffer from significantly longer runtime costs. Conversely, heuristic stitching methods, \textit{e.g}., CompDiffuser (CD), are fast but prone to lower success. \textbf{Ours} bridges this gap, achieving state-of-the-art performance within reasonable runtime. }
    \label{fig:teaser}
    \vspace{-0.3em}
\end{figure}

Despite this promise, deploying diffusion planners in realistic long-horizon settings remains challenging. While an ideal training set would contain demonstrations spanning complete task horizons so that the model can be queried with boundary constraints (\textit{e.g}., start and goal states) \citep{janner2022diffuser}, real robot datasets are typically \textit{fragmented} into short clips due to skill segmentation~\cite{kaelbling2011hierarchical, garrett2021integrated, mishra2023generative} or task-agnostic play \citep{lynch2020learning, agia2022stap, ogbench_park2024}. This creates a fundamental mismatch: we have abundant local motion statistics but few full episodes that cover the downstream task. Hence, the planner is forced to generate long trajectories of length $L$ using only short-horizon experience of length $l \ll L$, a form of \textbf{compositional generation} in which short trajectory “chunks” must be composed into a globally consistent and feasible plan \citep{ajay2023compositional,yang2023compositional}.

\begin{figure*}[!t]
    \centering
    \includegraphics[width=1\linewidth]{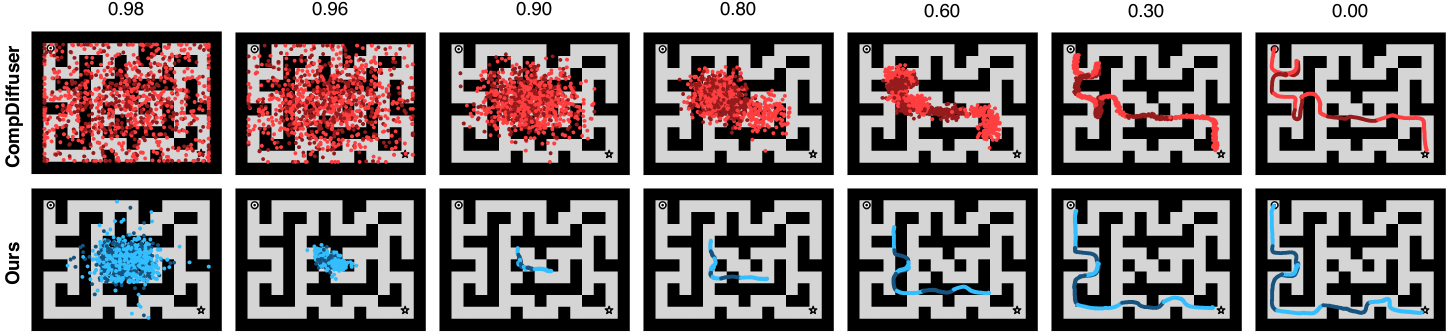}
    \caption{\textbf{Comparison of the denoising process.} Snapshots of the 2D trajectory on one OGBench AntMaze Giant sample during reverse diffusion process (timesteps are normalized to $1 \rightarrow 0$). While CompDiffuser (CD, top) exhibits scattered samples until the later steps, our method (bottom) rapidly converges to one possible global mode at earlier time ($t\approx0.8$ for ours and $t \approx 0.3$ for CD) and maintains this mode throughout the denoising process. The alternating darker colors denote different chunks. (\textit{Best viewed in color}.)}
    \label{fig:diff}
    \vspace{-0.3em}
\end{figure*}

To address compositional generation and create one long-horizon trajectory, recent methods rely on the stitching of the short trajectory chunks' predictions during the denoising process of diffusion sampling. Such approaches have employed simple mode-averaging~\cite{mishra2023generative, mishra2024generative} or \edited{boundary-conditioned inpainting} to generate local chunks, which are then stitched using \edited{simple stitching heuristics} \citep{lidiffstitch2024,luo2025compdiffuser} or hierarchical refinement \citep{chen2024hierarchicaldiffuser,chen2025hmdiffusion}. 
While effective for simple cases, we argue that these methods operate on a flawed inductive bias: they assume global consistency can be achieved by heuristically averaging or overwriting local overlaps, thereby mishandling the underlying mathematics of composition of diffusion scores.
A recent work, CDGS \citep{mishra2025compositional}, attempts to resolve such inconsistencies at inference time through resampling and scaling inference-time compute. While CDGS improves success rates over heuristic stitching methods, \textit{e.g}., CompDiffuser~(CD)~\citep{luo2025compdiffuser}, its iterative nature imposes substantial computational overhead. As illustrated in Fig.~\ref{fig:teaser}, a significant gap remains in the literature landscape: existing methods are either fast but inconsistent, or consistent but prohibitively slow.

In this paper, we identify a fundamental failure mode in existing stitching approaches. The issue is not \edited{using boundary conditions for local chunks}, but rather how overlapping updates are aggregated. Stitching heuristics, such as averaging or overwriting, \edited{ignore how the chunk's interior mean shifts in response to boundary-condition perturbations}. Mathematically, this omission drops a necessary chain-rule term: the gradient component that adjusts \edited{boundary variables based on ``feedback''} from the chunk's interior points. By ignoring this interaction, existing methods induce a \textbf{non-conservative} score field (\textit{i.e.}, with non-zero curl). Because this field cannot be expressed as the gradient of any global log-density function, the usual score-based reverse-diffusion interpretation no longer applies. Empirically, this results in slower convergence and/or a failure to settle into one possible trajectory mode, as shown in Fig.~\ref{fig:diff}.
Conversely, restoring conservativity \edited{of the induced compositional energy} provides a principled direction for mode-consistent denoising, thereby potentially reducing the need for heavy resampling required to enforce mode consistency, as in CDGS.

Guided by the above findings, we propose \textbf{an energy-based compositional diffuser}~(\ourmethod).
Our key insight is to invert the problem:
instead of composing predicted trajectories from local denoisers, we first construct a global, scalar-valued energy by summing local  energies induced by each chunk denoiser, and then \edited{use its negative gradient as a conservative correction direction.}
\edited{Because the correction} is obtained by differentiating a scalar-valued energy, the energy-derived field is conservative by construction.
This energy perspective also reveals what stitching approaches miss, a \textbf{\edited{boundary reaction}} term (a Jacobian-vector product) that propagates interior inconsistency back to chunk boundaries as boundary messages.
Finally, we show how to approximate this reaction efficiently with a single block-tridiagonal solve per chunk, giving per-denoising-step complexity linear in the planning horizon.

We summarize our main contributions:
\begin{itemize}[leftmargin=1em]
\itemsep0.2em
\item We identify a theoretical failure mode of heuristic stitching diffusion planners: the induced update field is generally non-conservative, and thus cannot correspond to the score of any global trajectory density.
\item We propose an energy-based formulation for compositional diffusion planners whose \edited{energy-derived correction field} is conservative by construction. With this formulation, we reveal a crucial \edited{boundary} reaction term that propagates interior inconsistency to chunk boundaries as messages.
\item We introduce a Markov-based approximation that computes this reaction efficiently via a block-tridiagonal solve, keeping inference time close to heuristic stitching approaches.
\end{itemize}
\vspace{-0.1em}
\begin{figure*}[t]
    \centering
    \includegraphics[width=1.0\linewidth]{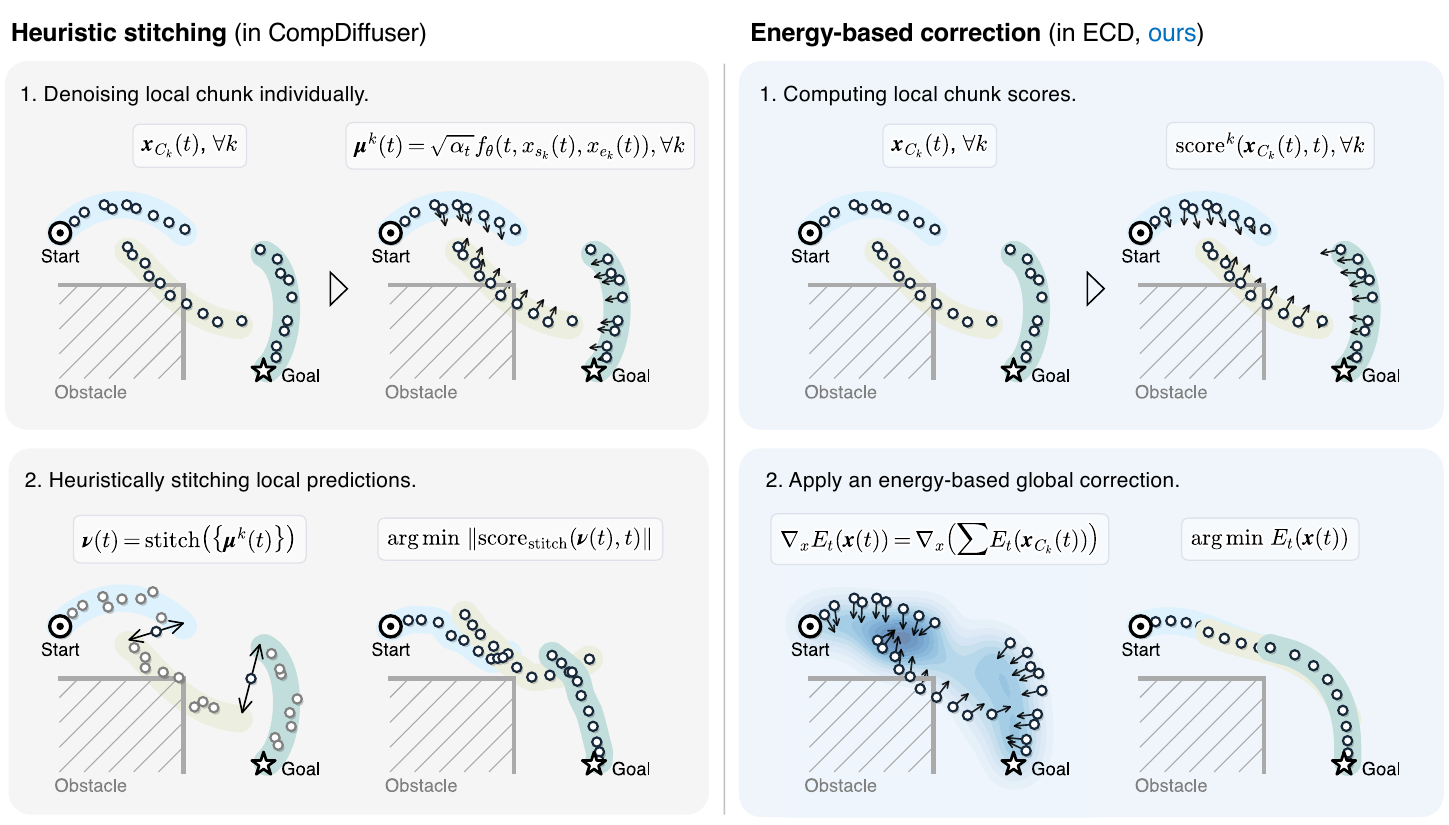}
    \vspace{-0.9em}
    \caption{\textbf{Method comparison.} At diffusion time $t$, the global trajectory $\bm{x}(t)$ is covered by overlapping chunks (colored bands) connecting a fixed start and goal around obstacles.
    \edited{For chunk $k$, $\bx_{s_k}(t)$ and $\bx_{e_k}(t)$ denote the boundary conditions given to the local denoiser (\textit{i.e.}, start/goal states or segments overlapped with neighboring chunks). \textbf{Left:} CompDiffuser denoises chunks independently, producing local noisy-mean predictions $\bmu^k(t)$, and then stitches them into a global mean $\bm{\nu}(t)$. However, the minimizer of this heuristically stitched score $\mathrm{score}_{\mathrm{stitch}}(\bm{\nu}(t),t)$ is generally not guaranteed to correspond to a globally consistent trajectory, due to its non-conservativity. \textbf{Right:} ECD constructs a scalar-valued global energy $E_t(\bx(t))=\sum_k E_t(\bx_{\mathcal C_k}(t))$ from local chunk energies and follows the negative energy gradient $-\nabla_{\bx}E_t(\bx(t))$. Therefore, ECD performs a conservative energy-based correction, and its denoising target is the global energy minimizer $\arg\min E_t(\bx(t))$. See Eq.~\ref{eq:stitch_score} for the stitched score definition.} (\textit{Best viewed in color.})}
    \label{fig:method_fig}
\end{figure*}

\section{Preliminaries}
\label{sec:prelim}

\subsection{Problem Setup}
\label{sec:problem_setup}

In a standard goal-conditioned planning setting, let $\mathcal{X} = \mathbb{R}^D$ denote the state space. Given a start state $x_s$ and a goal $x_g$, the objective is to generate a feasible trajectory $x_{0:L} \in \mathbb{R}^{(L+1) \times D}$ such that $x_0 = x_s$ and $x_L = x_g$. Following prior work \citep{mishra2023generative,luo2025compdiffuser}, we formulate planning as conditional generative modeling:
\begin{equation}
    \bx_{1:L-1} \sim p_\theta(\cdot \mid x_0 = x_s, x_L = x_g).
\end{equation}
Although we describe the method over states for clarity, the formulation extends directly to state-action sequences~\cite{janner2022diffuser}. In practice, diffusion planners often generate state sequences only and recover actions using a separate inverse-dynamics model that predicts the action given consecutive states~\citep{ajay2022conditional, luo2025compdiffuser}. Accordingly, we focus on state-sequence planning throughout the paper.

In this paper, we study the \emph{out-of-span} long-horizon setting where training clips are much shorter than the downstream planning horizon.
Specifically, the dataset $\mathcal{D}$ consists only of fragments of fixed length $l$,
$\mathcal{D} := \{ x^{(i)}_{s_i:s_i+l} \}_{i=1}^N$, where $l \leq L$. 
This reflects common real-robot pipelines where skill separation, resets, safety interventions, or teleoperation yield short segments rather than complete task episodes.
The core challenge is to learn local dynamics from the fragments in $\mathcal{D}$ and compose them consistently to solve horizon-$L$ tasks during inference.

\begin{figure*}[t]
    \centering
    \includegraphics[width=1\linewidth]{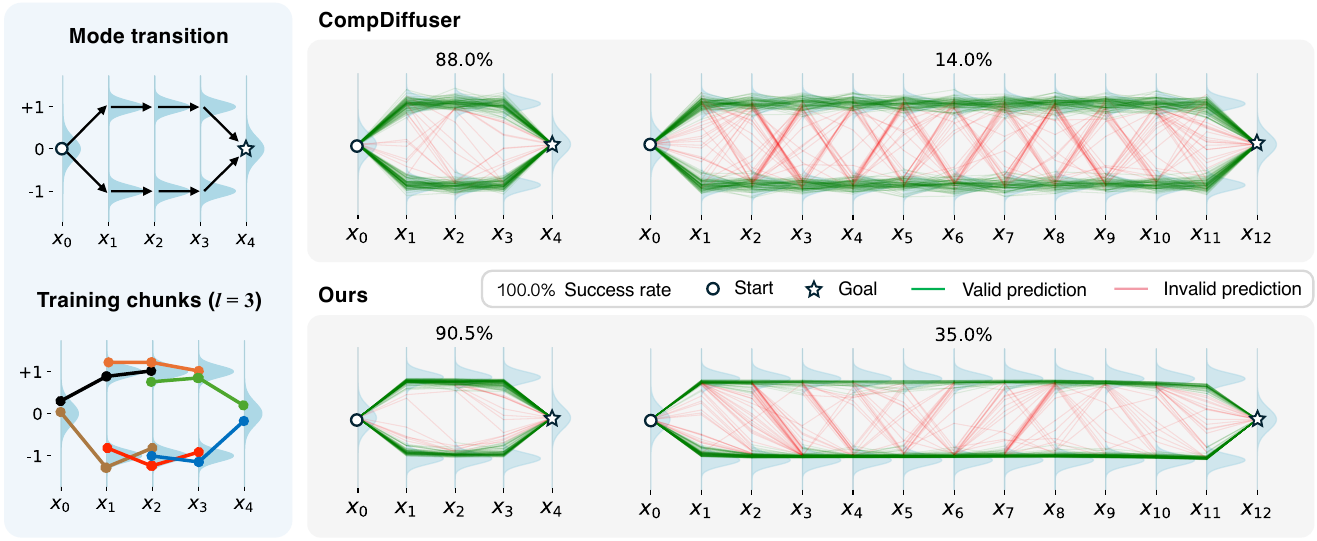}
    \caption{\textbf{Toy example with a sequence of 1D distributions}: For scalar variables $\{x_{0:L}\}$, there are two feasible long-horizon plans from start $x_0 = 0$ to goal $x_L = 0$: one through the top modes ($\mathbb{E}[x_i]=+1$) and one through the bottom modes ($\mathbb{E}[x_i]=-1$). Each training chunk contains $l=3$ consecutive variables. While both CD and our method show robust performance at $L=4$, CD's performance degrades significantly at $L=12$. Our method enforces better consistency in overlapping regions via an \edited{energy-derived conservative} score field, enabling improved long-horizon continuity while keeping individual segments in-distribution. In contrast, CD's stitching strategy leads to inconsistent jumps between modes. In gray blocks, each subfigure displays 200 runs for each method; \textcolor{OliveGreen}{{green}} lines denote valid trajectories, while \textcolor{Salmon}{{coral}} lines indicate invalid ones. The number above each subfigure is the success rate. (\textit{Best viewed in color}.)}
    \label{fig:toy}
    \vspace{-0.2em}
\end{figure*}

\subsection{Compositional Diffusion Planning}

CompDiffuser (CD) \citep{luo2025compdiffuser} first addresses this setting by combining a local conditional diffusion model with heuristic prediction stitching.

\textbf{Local chunk denoiser.} 
Following CD's convention, the task horizon $L$ is covered by $K+1$ overlapping chunks with length $l$ and stride $d$, where $l/2 < d < l$. 
For each chunk $k$, its start and end indices are $s_k := k d$ and $e_k := kd + l$, respectively. Without loss of generality, we assume $L=Kd+l$. We denote the $k$-th chunk index set as
$
C_k := \{s_k, \cdots, e_k\}  = \{kd, \cdots, kd+l\}.
$

Let $\mathbf{S}_k$ select the states in chunk $C_k$. We further introduce two local selectors: $\mathbf{P}_k$ selects the coordinates whose noisy mean is predicted or scored, and $\mathbf{O}_k$ selects the variables used as the boundary condition. Throughout the paper, a boundary condition can refer to the chunk's start/end states or to short boundary segments shared with neighboring chunks.

\edited{At diffusion timestep $t$, the denoiser estimates a clean local prediction for the selected coordinates. The network may receive the current noisy selected coordinates as an input, but the bridge-energy correction differentiates the prediction only through the boundary-condition channel. Equivalently, the selected-coordinate input is treated as part of the CD proposal, or as stop-gradient, when defining the correction energy. We write the corresponding noisy mean as}
\edited{\begin{equation}
\begin{aligned}
\bmu^k(t) & := \bmu^k\!\big(\mathbf{c}_k(t),t\big)\\
& := \sqrt{\bar\alpha_t}\,\hat\bx^k_{0,\theta}\!\big(t,\operatorname{sg}[\mathbf{P}_k\mathbf{S}_k\bx(t)]\mid \mathbf{c}_k(t)\big),
\label{eq:noisy-mean}
\end{aligned}
\end{equation}}
\edited{Here, 
$\mathbf{c}_k(t):=\mathbf{O}_k\mathbf{S}_k\bx(t)$ is the boundary condition,  $\bar\alpha_t$ is the cumulative diffusion drift and $\operatorname{sg}[\cdot]$ denotes stop-gradient for the energy correction. The crucial point is that the local mean $\bmu^k$ is differentiated as a function of the boundary condition. Therefore, when the boundary condition changes during reverse diffusion, the entire local prediction can shift.}

\paragraph{Inference via stitching of local chunks.}

During reverse diffusion, each local expert generates a noisy-mean prediction $\bmu^k(t)$ over its window $C_k$. Because indices in the overlapping regions $j \in C_k \cap C_{k+1}$ receive conflicting candidates from adjacent experts, CD resolves these conflicts via a heuristic {override} or {averaging} rule.

Let $\nu(t)\in\mathbb{R}^{(L+1)\times D}$ be the global stitched mean. Under stride assumption, the value at an overlap index $j \in C_k \cap C_{k+1}$ is composed as:
\begin{equation}
    \nu_j(t) = w_k \mu^{k}_j(t) +  (1-w_k)\mu^{k+1}_j(t).
    \label{eq:stitching}
\end{equation}
\edited{The corresponding stitched noisy-mean score estimate is}
\edited{\begin{equation}
    \mathrm{score}_{\mathrm{stitch}}(\bx,t)
    := \frac{\bm{\nu}(t)-\bx(t)}{\sigma_t^2}.
    \label{eq:stitch_score}
\end{equation}}
In CD's autoregressive (AR) variant,\footnote{CD offers both AR and interleaved variants in which chunks are denoised separately and \edited{boundary conditions} for chunk $k$ are drawn from neighboring chunks. Our analysis applies whenever the local mean depends on the boundary condition and the stitching rule omits the chain-rule derivative. Details are in Appendix~\ref{app:cd_nonconservative_full}.} \edited{the shared region is determined entirely by one direction}. For \texttt{ar\_forward}, $w_k=0$ (the later chunk overwrites the earlier); for \texttt{ar\_backward}, $w_k=1$ (the earlier chunk overwrites the later).
This stitched mean $\nu(t)$ is then used to compute the noise estimate $\hat{\epsilon}(t)$ for updating the global state.

\edited{Such heuristic composition is fast, but it uses only the local predictions themselves. It does not differentiate how a chunk prediction changes with respect to the boundary condition $\mathbf{O}_k\mathbf{S}_k\bx(t)$. The missing derivative is the boundary reaction term derived in the next section.}

\section{Energy-based Compositional Diffuser}

We approach compositional planning by \textit{inverting} the standard paradigm: rather than stitching local predictions to approximate a global trajectory, we construct a global energy landscape from local chunk denoisers and define the plan as one of the energy minimizers.
\edited{Because the correction direction is the negative gradient of a scalar field, the resulting energy-derived vector field is \emph{conservative} by construction \citep{marsden2003vector}. In the practical sampler, this field is used as a conservative compatibility \textbf{correction} to a CD-style interleaved sampler.}

\subsection{Properties of a Consistent Global Energy} \label{sec:energy_properties}

To build intuition, consider the 1D toy example in Fig.~\ref{fig:toy}. The task is to plan a trajectory $x_{0:L}$ from $x_0=0$ to $x_L=0$. The underlying data distribution is multimodal, favoring trajectories that stay entirely at $+1$ (top path) or $-1$ (bottom path). A valid trajectory must commit to one mode. However, an inpainting denoiser for \edited{a} short chunk ($l=3$) sees valid training data in both modes. The core challenge is ensuring that a decision to choose \edited{a} mode in one chunk propagates globally to align all other chunks, avoiding the ``zig-zag'' failure where the trajectory incoherently jumps between $+1$ and $-1$, denoted as pink lines in Fig.~\ref{fig:toy}.

To solve the mode-consistency illustrated above, a global energy function $E(\bx)$ should satisfy \edited{the} following properties:

\begin{enumerate}[leftmargin=1.2em]
\itemsep0.2em 
\item \textbf{Interior update.} For any given chunk, if the \edited{boundary variables} are fixed (\textit{e.g}., at mode $+1$), the energy minimum for the interior nodes must lie on the feasible manifold connecting them. This ensures local kinematic validity and avoids collision.
\item \textbf{\edited{Reaction to boundary conditions.}} \edited{The energy must couple the interior to the boundary conditions. If the interior is high-energy (\textit{e.g.}, in the middle of two modes), it must exert a ``reaction force'' on the boundary conditions to move them toward feasible modes.}
\item \textbf{Chunk consensus.} Since chunks overlap, the global energy needs to be additive, forcing the overlapping nodes to satisfy the constraints of both the left and right overlapping simultaneously.
\end{enumerate}

We illustrate the above three properties in Fig.~\ref{fig:energy_prop} with the desired gradient direction of each property.
In addition to these, the energy should have consistent {noise scaling} throughout the diffusion time steps, so its contribution has a magnitude compatible with DDIM at different times.

\begin{figure}[t]
    \centering
    \includegraphics[width=1.005\linewidth]{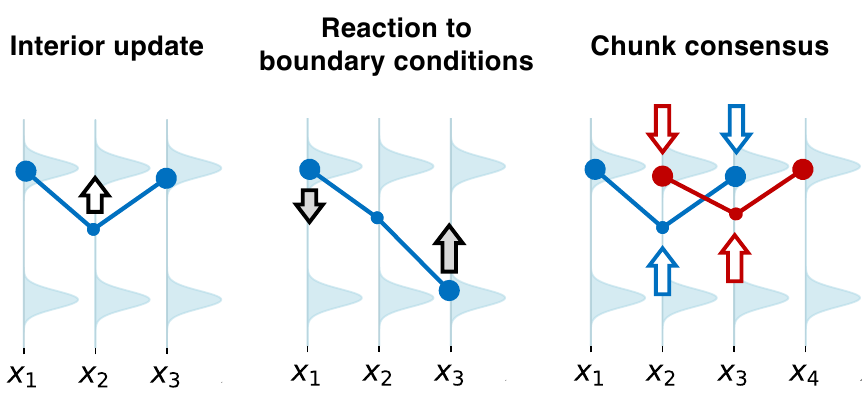}
    \caption{\textbf{Physical interpretation of three properties.} The arrows indicate the desired gradient direction leading to energy minimization for each property. \edited{Here, small dots denote interior variables and large dots denote boundary variables.} (\textit{Best viewed in color}.)}
    \label{fig:energy_prop}
    \vspace{-0.3em}
\end{figure}

\subsection{Energy-based \edited{Correction}}

\label{sec:model}
\edited{We realize these properties by lifting the conditional predictions of the trained denoiser into an energy-based correction.}
Recall chunk \(k\) covers indices \(C_k=\{s_k,\dots,e_k\}\) and we denote \(\bx_k := \mathbf{S}_k\bx\in\mathbb{R}^{(l+1)D}\).
\edited{Let $\mathbf{P}_k\bx_k$ denote the local coordinates that we score, and let $\mathbf{O}_k\bx_k$ denote the local boundary condition. The selector $\mathbf{P}_k$ selects the local trajectory coordinates being corrected, while $\mathbf{O}_k$ selects the boundary variables that condition the chunk prediction.}
\edited{Following the stop-gradient convention in Eq.~\ref{eq:noisy-mean}, at diffusion step $t$ the trained denoiser provides a conditional mean prediction for the selected local coordinates:}
\edited{\begin{equation}
\bmu^k(t) := \bmu^k(\mathbf{O}_k \bx_k(t),t)\in\mathbb{R}^{d_k},
\quad d_k:=\dim(\mathbf{P}_k \bx_k).
\end{equation}}
Here, all derivatives of $\bmu^k$ below are with respect to the boundary condition $\mathbf{O}_k \bx_k$, and the denoiser's direct noisy-input argument is held fixed in this bridge-energy interpretation.

\paragraph{Local energy.} We define the local energy of a single chunk as the squared Mahalanobis distance between the current selected local state and this predicted mean:
\edited{\begin{equation}
E_{\text{bridge}}^{k}(\bx,t)
\;:=\;
\frac{1}{2\sigma_t^2}
\left\|
\mathbf{P}_k\bx_k - \bmu^k\!\big(\mathbf{O}_k\bx_k,t\big)
\right\|^2_{W_k},
\label{eq:bridge_energy}
\end{equation}}
where \(\|\bv\|^2_{W_k} := \bv^T W_k \bv\), and $W_k$ is a diagonal weighting matrix handling overlap redundancy, \textit{e.g}., \(w_j^k = 1/|\{k': j\in C_{k'}\}|\).
\edited{In the context of Fig.~\ref{fig:toy}, if a chunk's boundary variables are at $+1$,} the expert prediction $\bmu^k$ will likely be a sequence of $+1$s. This energy term penalizes any deviation of the selected local nodes from this mode.

\paragraph{Global energy.}
We define the global energy as a sum of local bridge energies:
\begin{equation}
E_t(\bx) \;:=\; \sum_{k=0}^{K} E_{\text{bridge}}^{k}(\bx,t).
\label{eq:global_energy}
\end{equation}
\edited{The induced bridge-model score/correction direction is the negative gradient of this energy, \textit{i.e}., $\mathrm{score}(\bx,t) := -\nabla_{\bx} E_t(\bx)$.}

\paragraph{Property satisfaction.} To see how the three properties are enforced, we first analyze the gradient of a single chunk (Eq.~\ref{eq:reaction_score_clean}) and then the summation over chunks (Eq.~\ref{eq:energy_to_score}).
Let
\edited{\begin{equation*}
\br_k(t) \;:=\; W_k\Big(\mathbf{P}_k\bx_k(t) - \bmu^k\big(\mathbf{O}_k\bx_k(t),t\big)\Big)
\end{equation*}}
be the residual error of the selected local coordinates. 
The score contribution of chunk $k$ is
\edited{\begin{equation}
\begin{aligned}
& \mathrm{score}^k(\bx,t)
= -\nabla_{\bx}\,E_{\text{bridge}}^{k}(\bx,t)\\
& =
-\frac{1}{\sigma_t^2}\,\mathbf{S}_k^T \underbrace{\mathbf{P}_k^T \br_k(t)}_{\text{Interior update}}
\;+\;
\frac{1}{\sigma_t^2}\,\mathbf{S}_k^T  \underbrace{\mathbf{O}_k^T \mathbf{J}_{O,k}(t)^T \br_k(t)}_{\text{Boundary reaction}}.
\end{aligned}
\label{eq:reaction_score_clean}
\end{equation}}
Here, \(\mathbf{J}_{O,k}(t) := \frac{\partial\,\bmu^k(\mathbf{O}_k\bx_k,t)}{\partial\,\mathbf{O}_k\bx_k}\) is the Jacobian of the local predictor with respect to its \edited{boundary condition}. \edited{If one instead differentiates through the noisy selected-coordinate input of the neural denoiser, Eq.~\ref{eq:reaction_score_clean} gains an additional self-Jacobian factor on the interior term. ECD uses the boundary-only/stop-gradient bridge energy in Eq.~\ref{eq:bridge_energy}, so Eq.~\ref{eq:reaction_score_clean} is the exact gradient of the energy being minimized.}
Summing over chunks gives
\begin{equation}
\mathrm{score}(\bx,t) = \sum_{k=0}^{K} \mathrm{score}^k(\bx,t).
\label{eq:energy_to_score}
\end{equation}
\edited{By construction, the ECD correction field $\mathrm{score}(\bx,t)=-\nabla_{\bx}E_t(\bx)$ is conservative and integrable as the score of the induced bridge-energy compatibility density, proportional to $\exp(-E_t(\bx))$. In the full implementation, this field is applied as an additive correction to the CD proposal.}
The terms in Eq.~\ref{eq:reaction_score_clean}, combined with the summation in Eq.~\ref{eq:energy_to_score}, satisfy the physical requirements as:
\begin{itemize}[leftmargin=1em]
\itemsep0.5em
    \item The \textbf{interior update} term pulls the selected local coordinates toward the predicted mean $\bmu^k$ (Property 1).
    \item \edited{The \textbf{boundary reaction} term captures how the local prediction $\bmu^k$ changes as the boundary condition $\mathbf{O}_k\bx_k$ moves. This term propagates the residual from the predicted coordinates back to the boundary variables (Property 2).}
    \item The \textbf{chunk consensus} is enforced by the summation of these local scores in Eq.~\ref{eq:energy_to_score}. Because the global energy is additive, any state variable $x_j$ belonging to multiple overlapping chunks receives gradient updates from all of them simultaneously. This drives the trajectory toward a configuration that minimizes the total energy of all overlapping chunks (Property 3).
\end{itemize}

\subsection{Why Heuristic Stitching is Non-Conservative}
\label{sec:non-integrability}

We can now formally pinpoint why heuristic stitching (\textit{e.g.} CD) fails more frequently in the multimodal data in Fig.~\ref{fig:toy}.
When CD resolves a conflict between a chunk that predicts $+1$ at an index and a neighbor that predicts $-1$ at that index, it computes a weighted average or overwriting (Eq.~\ref{eq:stitching}). 
\edited{This corresponds to a bridge-only approximation that uses the residual term while effectively setting the reciprocal boundary-condition Jacobian \(\mathbf{J}_{O,k}(t)\) in Eq.~\ref{eq:reaction_score_clean} to zero.}

We formalize this in Proposition~\ref{prop:nonconservative} (proved in Appendix~\ref{app:cd_nonconservative_full}), which states formally why CD-style stitching results in a non-conservative vector field. \edited{Such a field cannot represent the gradient of any valid log-density, making the standard score-based interpretation invalid and helping explain the observed mode drift.}

\begin{proposition}[Non-conservative score field of chunk stitching]
\label{prop:nonconservative}
Assume there exists a chunk $k$ and an index $j$ for which the stitched mean uses $\mu^k_j$ with nonzero weight and $\mu^k_j(\bx,t)$ depends nontrivially on a chunk \edited{boundary condition}, i.e.,
\edited{\[
\frac{\partial \mu^k_j}{\partial (\mathbf{O}_k\bx_k)_q} \neq 0
\]}
for some coordinate $q$, at some $\bx$ and fixed $t$. \edited{Assume also that this dependence is not exactly canceled by other terms contributing to the same stitched coordinate.}
Further assume that stitching assigns that \edited{boundary-conditioning} coordinate using only chunk predictions that do \emph{not} depend on $x_j$.
Then, the Jacobian $\nabla_{\bx}\mathrm{score}_{\text{stitch}}(\bx,t)$ is asymmetric.
\end{proposition}

\subsection{Efficient Approximation via Markovian Structure}
\label{sec:efficient_score}
Calculating the exact Jacobian \(\mathbf{J}_{O,k}(t)\) in Eq.~\ref{eq:reaction_score_clean} requires differentiating through the neural network at each denoising step, which is computationally prohibitive.
However, in robotic planning tasks, the dependency structure is often \edited{locally Markovian, where} $x_i$ depends primarily on $x_{i-1}$ and $x_{i+1}$.
We exploit this to approximate \(\mathbf{J}_{O,k}^T\br_k\) by solving a sparse, block-tridiagonal linear system.
This reduces the complexity of computing the reaction term from quadratic to linear in the horizon length, $O(L)$, enabling the efficient inference shown in Algorithm~\ref{alg:impl_ecd_ddim}.
Details are in Appendix~\ref{app:approx_close_exact}.

\begin{table}[t]
\centering
\adjustbox{scale=1.0,center}
{
\footnotesize
\setlength{\tabcolsep}{4pt}

}

\vspace{4pt}
\adjustbox{scale=0.97,center}
{
\footnotesize
\begin{tabular}{ll ccc}
\toprule
\multirow{2}{*}{\textbf{Type}} & \multirow{2}{*}{\textbf{Method}} & \multicolumn{3}{c}{\textbf{Success Rate [\%]}} \\
\cmidrule(lr){3-5}
&  & \texttt{Medium} & \texttt{Large} & \texttt{Giant} \\
\midrule
\multirow{6}{*}{\textit{Offline RL}} & GCBC & \valstd{23}{18} & \valstd{7}{5} & \valstd{0}{0} \\
& GCIVL & \valstd{70}{14} & \valstd{12}{6} & \valstd{0}{0} \\
& GCIQL & \valstd{21}{9} & \valstd{31}{2} & \valstd{0}{0} \\
& CRL & \valstd{0}{1} & \valstd{0}{0} & \valstd{0}{0} \\
& QRL & \valstd{80}{12} & \valstd{84}{15} & \valstd{50}{8} \\
& HIQL & \valstd{74}{6} & \valstd{13}{6} & \valstd{0}{0} \\
\midrule
\multirow{4}{*}{\textit{Diffusion}} & GSC & \valstd{\mathbf{100}}{0} & \valstd{\mathbf{100}}{0} & \valstd{39}{1} \\
& {CompDiffuser} & \valstd{\mathbf{100}}{0} & \valstd{\mathbf{100}}{0} & \valstd{{77}}{3} \\
& {CDGS} & \valstd{\mathbf{100}}{0} & \valstd{97}{2} & \valstd{69}{4} \\
& \textbf{Ours} & \valstd{\mathbf{100}}{0} & \valstd{\mathbf{100}}{0} & \valstd{\textbf{84}}{2} \\
\bottomrule
\end{tabular}
}
\vspace{0.5em}
\caption{
\textbf{Quantitative results on PointMaze-stitching datasets in \citet{ghugare2024closing}.} 
We compare \ourmethod{} to baselines of multiple categories, including diffusion, data augmentation, and offline reinforcement learning in \citet{ghugare2024closing}, GSC~\citet{mishra2023generative}, CompDiffuser~\citet{luo2025compdiffuser} and CDGS~\citet{mishra2025compositional}.
Results are averaged  over 3 seeds and standard deviations are shown after the $\pm$ sign.
}
\label{tab:pointmaze_ben_ogb}
\vspace{-1.2em}
\end{table}

\begin{table*}[t!]
\centering
\adjustbox{scale=1.0,center}
{
\footnotesize
\setlength{\tabcolsep}{5pt}
\begin{tabular}{ccc ccc ccc cccc}
\toprule
\textbf{Environment} & \textbf{Type} & \textbf{Size} & \textbf{GCBC} & \textbf{GCIVL} & \textbf{GCIQL} & \textbf{QRL} & \textbf{CRL} & \textbf{HIQL} & \textbf{GSC} & \textbf{CD} & \textbf{CDGS} & \textbf{Ours} \\
\midrule

\multirow[c]{3}{*}{\textbf{AntMaze}}

 & \multirow[c]{3}{*}{\texttt{{Stitch}}} 

 & \texttt{Medium} & $45$ {\tiny $\pm 11$} & $44$ {\tiny $\pm 6$} & $29$ {\tiny $\pm 6$} & $59$ {\tiny $\pm 7$} & $53$ {\tiny $\pm 6$} & ${94}$ {\tiny $\pm 1$} 
 & \valstd{ \textbf{98} } { 1 }
 & \valstd{ {95} } { 4 } & \valstd{ {96} } { 1 }  & \valstd{ \underline{97} } { 1 }  \\

 &  & \texttt{Large} & $3$ {\tiny $\pm 3$} & $18$ {\tiny $\pm 2$} & $7$ {\tiny $\pm 2$} & $18$ {\tiny $\pm 2$} & $11$ {\tiny $\pm 2$} & ${67}$ {\tiny $\pm 5$} 
 & \valstd{ \underline{77} } { 4 }
 & \valstd{ {74} } { 4 } & \valstd{ {71} } { 3 } & \valstd{ \textbf{82} } { 1 } \\

 &  & \texttt{Giant} & $0$ {\tiny $\pm 0$} & $0$ {\tiny $\pm 0$} & $0$ {\tiny $\pm 0$} & $0$ {\tiny $\pm 0$} & $0$ {\tiny $\pm 0$} & ${21}$ {\tiny $\pm 2$} 
 & \valstd{ {36} } { 1 }
& \valstd{ {72} } { 9 } & \valstd{ \underline{78} } { 7 } & \valstd{ \textbf{82} } { 6 }\\

\cmidrule{1-13}

\multirow[c]{3}{*}{\makecell{\textbf{Humanoid-}\\\textbf{Maze}}} 

 & \multirow[c]{3}{*}{\texttt{{Stitch}}} 
 & \texttt{Medium} & $29$ {\tiny $\pm 5$} & $12$ {\tiny $\pm 2$} & $12$ {\tiny $\pm 3$} & $18$ {\tiny $\pm 2$} & $71$ {\tiny $\pm 3$} & ${\textbf{96}}$ {\tiny $\pm 4$} 
 & \valstd{ {90} } { 2 }
 & \valstd{ {90} } { 4 }
  & \valstd{ {91} } { 1 } & \valstd{ \underline{92} } { 2 }
 \\
 &  & \texttt{Large} & $6$ {\tiny $\pm 3$} & $1$ {\tiny $\pm 1$} & $0$ {\tiny $\pm 0$} & $3$ {\tiny $\pm 1$} & $6$ {\tiny $\pm 1$} & ${31}$ {\tiny $\pm 3$}
 & \valstd{ {58} } { 5 }
 & \valstd{ \underline{59} } { 2 }
 & \valstd{ {50} } { 3 }
 & \valstd{ \textbf{64} } { 4 }
 \\
 &  & \texttt{Giant} & $0$ {\tiny $\pm 0$} & $0$ {\tiny $\pm 0$} & $0$ {\tiny $\pm 0$} & $0$ {\tiny $\pm 0$} & $0$ {\tiny $\pm 0$} & ${12}$ {\tiny $\pm 2$} 
 & \valstd{ {7} } { 1 }
 & \valstd{ \underline{{42}} } { 1 }
 & \valstd{ {23} } { 6 }
 & \valstd{ \textbf{49} } { 1 }
 \\

\bottomrule
\end{tabular}
}
\vspace{0.5em}
\caption{
\textbf{Success rates  [\%] on AntMaze and HumanoidMaze in OGBench.} 
We benchmark our method on the 5 test-time tasks defined in OGBench with 20 episodes per task. 
Results are averaged over 3 seeds and standard deviations are shown after the $\pm$ sign.
}
\label{table:ogb_ant_humanoid}

\end{table*}

\begin{table*}[t]

\centering
\adjustbox{scale=0.97,center}
{
\footnotesize
\setlength{\tabcolsep}{5pt}
\begin{tabular}{ccc ccc ccc ccc cccc}
\toprule
\textbf{Environment} & \textbf{Size} & \textbf{GCBC} & \textbf{GCIVL} & \textbf{GCIQL} & \textbf{QRL} & \textbf{CRL} & \textbf{HIQL}  & \makecell{\textbf{GSC} \\ (4D)} &  \makecell{\textbf{CD} \\ (4D)}  
& \makecell{\textbf{GSC} \\ (17D)} & \makecell{\textbf{CD} \\ (17D)} & \makecell{\textbf{CDGS} \\ (17D)} & \makecell{\textbf{Ours} \\ (17D)}\\
\midrule
\multirow[c]{2}{*}{\makecell{\textbf{AntSoccer} }} 
 
 & \texttt{Arena} & ${34}$ {\tiny $\pm 4$} & $21$ {\tiny $\pm 3$} & $5$ {\tiny $\pm 2$} & $2$ {\tiny $\pm 1$} & $2$ {\tiny $\pm 1$} & $23$ {\tiny $\pm 2$} 
 & \valstd{ {41} } { 4 }
 & \valstd{ {55} } { 6 } 
 & \valstd{ {65} } { 3 }
 & \valstd{ {61} } { 3 }
  & \valstd{ {65} } { 3 }
  & \valstd{ \textbf{66} } { 2 } \\
 & \texttt{Medium} & $2$ {\tiny $\pm 1$} & $1$ {\tiny $\pm 0$} & $0$ {\tiny $\pm 0$} & $0$ {\tiny $\pm 0$} & $0$ {\tiny $\pm 0$} & ${8}$ {\tiny $\pm 2$} 
 & \valstd{ {5} } { 2 } 
 & \valstd{ {13} } { 1 } 
 & \valstd{ {12} } { 2 } 
 & \valstd{ {17} } { 3 }
& \valstd{ {9} } { 3 }
 & \valstd{ \textbf{20} } { 2 } \\

\bottomrule
\end{tabular}
}
\vspace{0.5em}
\caption{
\textbf{Success rates [\%] on high-dimensional AntSoccer in OGBench.}
We evaluate two AntSoccer environments with 17D state dimensions. Results are averaged over 3 seeds and standard deviations are shown after the $\pm$ sign.
}
\label{table:ant_soccer}
\vspace{-0.5em}
\end{table*}

\section{Experiments}

In this section, we evaluate \ourmethod's performance on the long-horizon trajectory stitching benchmark from OGBench~\cite{ogbench_park2024}. For all experiments, we consider the input as either the 2D position of the center of mass or a higher-dimensional state that also includes the joint poses of the agent. 

\textbf{Baselines.} We compare the performance of our method against baselines that perform heuristic stitching (CD)~\cite{luo2025compdiffuser}, mode averaging (GSC)~\cite{mishra2023generative} and compute scaled mode averaging (CDGS)~\cite{mishra2025compositional}. In addition, we borrow the goal conditioned reinforcement learning benchmark from OGBench~\cite{ogbench_park2024} comprising of goal-conditioned behavioral cloning (GCBC)~\cite{lynch2020learning,ghosh2019learning},
goal-conditioned implicit V-learning (GCIVL) and Q-learning (GCIQL)~\cite{kostrikov2021offline},
Contrastive RL (CRL)~\cite{eysenbach2022contrastive}, 
Quasimetric RL (QRL)~\cite{wang2023optimal}, 
and Hierarchical implicit Q-learning (HIQL)~\cite{park2024hiql}.

\begin{figure}
    \centering
    \includegraphics[width=0.99\linewidth]{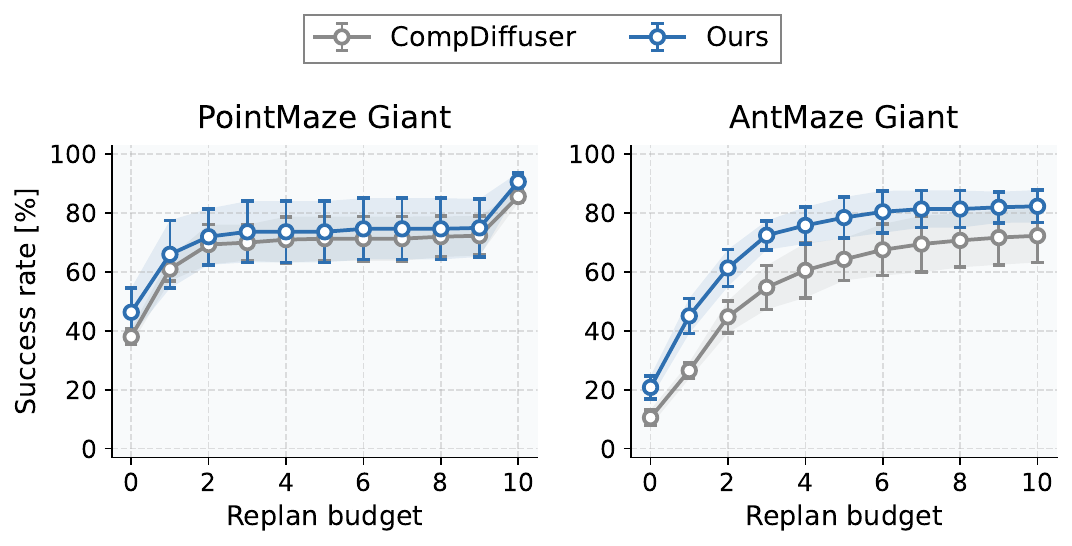}
    \caption{\textbf{Performance at different replan budgets}. We compare the success rates of our method and CD on PointMaze Giant and AntMaze Giant as a function of the \edited{allowed number of replans}. Our method consistently outperforms CD across all replan budgets.}
    \label{fig:replan}
    \vspace{-0.3em}
\end{figure}

\subsection{Effectiveness Study}

\textbf{\ourmethod{} effectively solves trajectory stitching.} We conduct experiments on the PointMaze domain where the short-horizon diffusion model is trained to predict a sequence of $x$-$y$ positions. At inference, \ourmethod{} generates plans for unseen start and goal pairs with horizons much beyond the training data. As shown in~\autoref{tab:pointmaze_ben_ogb}, \edited{our method matches the baselines} for \texttt{Medium} and \texttt{Large} mazes. With increasing plan length requirements for the \texttt{Giant} maze, \edited{the bridge plus boundary-reaction update promotes coherent boundary consistency and substantially improves performance over baselines}. We conduct similar experiments with AntMaze and HumanoidMaze, where the trajectory planning is done to get the sequence of $x$-$y$ positions that is further executed with the help of an inverse dynamics model~\cite{ajay2022conditional, luo2025compdiffuser}. As illustrated in~\autoref{table:ogb_ant_humanoid}, we observe consistent improvement in performance across all mazes when compared to heuristic stitching and simple model averaging baselines.

\begin{table}[h]
    \centering
    \footnotesize
    \begin{tabular}{ccc ccc}
    \toprule
       \multirow{1}{*}{\textbf{Environment}} & \multirow{1}{*}{\textbf{Size}}  & \multirow{1}{*}{\textbf{Method}} & \textbf{2D} & \textbf{15D} \\
       \midrule
        \multirow{4}{*}{\textbf{AntMaze}} & \multirow{2}{*}{\texttt{Large}} &  CD & \valstd{{74}}{4} & \valstd{83}{3} \\
      &  &  Ours & \valstd{\textbf{82}}{1} & \valstd{\textbf{89}}{1} \\
            \cmidrule(lr){2-5}
        \ & \multirow{2}{*}{\texttt{Giant}} &  CD & \valstd{{72}}{9} & \valstd{41}{3} \\
      &  &  Ours & \valstd{\textbf{82}}{6} & \valstd{\textbf{51}}{3} \\
        \bottomrule
    \end{tabular}
    \vspace{0.5em}
    \caption{\textbf{Success rates  [\%] on high-dimensional AntMaze stitching tasks}. We compare our method against CompDiffuser (CD) on AntMaze Large and Giant with 15-dimensional state space. Our energy-based approach consistently outperforms CD, particularly in the more challenging giant environment. }
    \label{tab:high_dim_antmaze}
    \vspace{-0.3em}
\end{table}

\textbf{\ourmethod{} improves performance in high-dimensional regime.} We further evaluate the performance in high-dimensional trajectory stitching scenarios, where we incorporate the joint positions along with the $x$-$y$ position of the center of mass of the agents. The algorithms directly synthesize a sequence of high-dimensional states. We show results on 2D and 15D AntMaze in~\autoref{tab:high_dim_antmaze}. Although success rates decrease with \edited{higher-dimensional planning} due to added complexity, \ourmethod{} consistently outperforms CD. We finally evaluate \ourmethod{} on AntSoccer, where the task is to synthesize a plan for an Ant agent to reach a soccer ball and push it to the provided goal. We \edited{evaluate} against OGBench standards across two maze environments, \texttt{Arena} and \texttt{Medium}, as shown in \autoref{table:ant_soccer}. We test under 17D configuration (tracking ant and ball positions plus all ant joint states). \edited{In both environments, \ourmethod{} achieves the highest mean success among the compared methods, suggesting that our formulation extends to joint dynamics and multi-object planning scenarios.}

\subsection{Ablation Study}

\textbf{\ourmethod{} improves the feasibility of long-horizon plans.} In our evaluation setup, we try to track the planned sequence of states using an inverse dynamics~(ID) model. Each time, we (i) execute actions based on ID outputs and (ii) progressively update the target state for the inverse dynamics model to be the immediate next planned state in the sequence. If the distance between the current state and the target state crosses a threshold, we assume that the agent cannot track the sequence anymore and \edited{fall back on replanning a new sequence of states to the goal.} An algorithm that outputs more feasible long-horizon plans will require \edited{fewer replans} to achieve the goal. As illustrated in ~\autoref{fig:replan}, given a fixed \edited{maximum number of replans}, we observe that \ourmethod{} exceeds CD's performance \edited{suggesting that our energy formulation not only enforces feasible boundary consistency but also keeps the local segments in-distribution after stitching.}

\begin{figure}[t]
    \centering
    \includegraphics[width=1\linewidth]{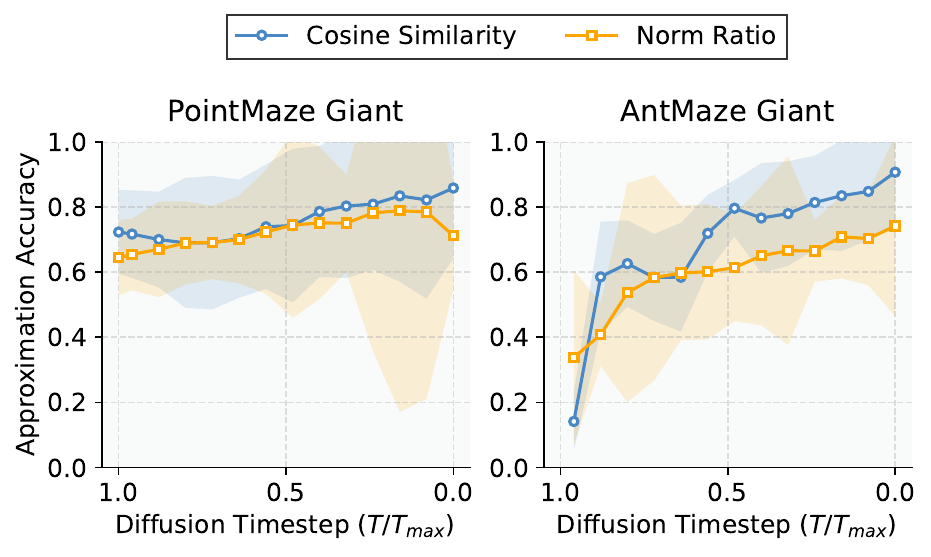}
    \caption{\textbf{Fidelity of reaction term approximation.} We compare our approximated reaction term with the exact JVP computed via backpropagation. We report the Cosine Similarity and Norm Ratio across normalized diffusion timesteps ($1.0 \to 0.0$) on PointMaze and AntMaze Giant. The approximation shows reasonable alignment, particularly in the later denoising stages. Shaded regions indicate standard deviation. }
    \label{fig:approx_acc}
    \vspace{-0.3em}
\end{figure}

\edited{\textbf{\ourmethod{} performs better with the exact boundary reaction term.}} 
To assess the effectiveness of our score approximation (\cref{sec:efficient_score}), we compare it against the \edited{exact boundary reaction term computed via automatic differentiation through the chunk denoiser}. As shown in~\autoref{tab:approx}, the exact computation yields slightly higher success rates on both PointMaze ($92\%$ vs. $91\%$) and AntMaze ($89\%$ vs. $87\%$). \edited{These results are consistent with the theoretical derivation of the boundary reaction term} and suggest that precise gradient feedback further enhances trajectory consistency. However, the exact calculation increases inference runtime by almost $3\times$ (approx. $22$s vs. $8$s). \edited{Our approximation thus provides a favorable trade-off, retaining most benefits of the energy-derived correction while ensuring efficiency.}

To further evaluate the approximation fidelity, we plot the approximation accuracy statistics (cosine similarity
and norm ratio) between the approximate and exact reaction terms in~\autoref{fig:approx_acc}. We observe that the accuracy of the approximation improves significantly as reverse diffusion proceeds ($t \to 0$). \edited{This trend is consistent with the theoretical error bound in Appendix~\ref{app:reaction_error_bound}, which predicts that the message approximation error scales with the noise level $\sigma_t$ across timesteps.}
\begin{table}[]
    \centering
    \footnotesize
    \begin{tabular}{cc ccc}
    \toprule
       \multirow{2}{*}{\textbf{Environment}}  & \multirow{2}{*}{\textbf{Compute}} & \textbf{Success Rate} & \textbf{Runtime} \\
      & & [\%] & [sec / episode] \\ 
       \midrule
        \multirow{2}{*}{\textbf{PointMaze}}  &  \texttt{Exact} & \valstd{\textbf{86}}{2} & \valstd{17}{2} \\
      &  \texttt{Approx} & \valstd{84}{2} & \valstd{\textbf{8}}{2} \\
        \midrule
        \multirow{2}{*}{\textbf{AntMaze}}  &  \texttt{Exact} & \valstd{\textbf{82}}{5} & \valstd{27}{17} \\
      &  \texttt{Approx} & \valstd{82}{6} & \valstd{\textbf{25}}{19} \\
        \bottomrule
    \end{tabular}
    \vspace{0.5em}
    \caption{\textbf{Approximate vs. exact reaction term}. \edited{We evaluate success rate and inference time on PointMaze and AntMaze Giant. Exact mode slightly improves success rate but is substantially slower at inference.}}
    \label{tab:approx}
    \vspace{-0.8em}
\end{table}

\paragraph{Inference runtime across methods.}
\autoref{tab:runtime} reports the planning runtime per episode for all four diffusion stitching methods on the two Giant environments (PointMaze and AntMaze), measured on a single NVIDIA RTX~5090 under the matched evaluation protocol of Fig.~\ref{fig:teaser}. \ourmethod{} runs at a comparable cost, adding only a moderate overhead on AntMaze from the block-tridiagonal solve. Notably, CDGS is excessively slower, since each reverse step performs multiple resampling rounds. \ourmethod{} attains the best success rate at a runtime close to CD and far below CDGS.

\begin{table}[t]
    \centering
    \footnotesize
    \setlength{\tabcolsep}{4.5pt}
    \begin{tabular}{l cc cc}
    \toprule
       \multirow{2}{*}{\textbf{Method}}  & \multicolumn{2}{c}{\textbf{PointMaze Giant}} & \multicolumn{2}{c}{\textbf{AntMaze Giant}} \\
       \cmidrule(lr){2-3} \cmidrule(lr){4-5}
       & SR [\%] & Runtime [s] & SR [\%] & Runtime [s] \\
       \midrule
       GSC              & $39$ & \valstd{5}{0}   & $36$ & \valstd{20}{9} \\
       CD               & $77$ & \valstd{8}{2}   & $72$ & \valstd{17}{10} \\
       CDGS ($4\times$)   & $69$ & \valstd{14}{4}  & $78$ & \valstd{46}{25} \\
       CDGS ($8\times$)   & $68$ & \valstd{29}{9}  & $82$ & \valstd{168}{39}\\
       \textbf{\ourmethod{} (Ours)} & $\mathbf{84}$ & \valstd{8}{2} & $\mathbf{82}$ & \valstd{25}{19} \\
       \bottomrule
    \end{tabular}
    \vspace{0.5em}
    \caption{\textbf{Inference runtime on the Giant environments.} Success rate [SR, \%] and planning runtime per episode [sec] (mean $\pm$ std over the episodes), measured on a single RTX~5090 for planner only. \ourmethod{} attains the highest success rate at a runtime comparable to CompDiffuser (CD) and far below CDGS at both $4\times$ and $8\times$ resampling rounds.}
    \label{tab:runtime}
\end{table}

\section{Prior Work}

\paragraph{Diffusion models for planning.} Diffusion models \cite{sohl2015deep, ho2020denoising} have been widely adopted for generative planning \cite{janner2022planning, ajay2022conditional, chi2023diffusionpolicy, wang2022diffusion, chen2024hierarchicaldiffuser}. Such methods formulate trajectory generation as probabilistic inference and have been applied to motion planning \cite{carvalho2023motion, luo2024potential} and task planning \cite{yang2023planning, fang2024dimsam}. Recent advancements integrate diffusion with hierarchical planning \cite{li2023hierarchical, chen2024simple}, uncertainty-aware sampling \cite{sun2024conformal}, and online replanning \cite{zhou2024adaptive}. However, these methods are generally limited to planning horizons as seen in training data. We build on compositional planning \cite{mishra2023generative, luo2025compdiffuser} to resolve the ``out-of-span" problem, and extend planning horizons significantly beyond training regime.

\paragraph{Trajectory stitching for long-horizon planning.} Effectively using fragmented trajectories to plan for long-horizon tasks allows reusing offline information to solve unseen tasks at inference time. Some effective approaches include using generative models \cite{lidiffstitch2024,lu2024synthetic,kim2024stitching,lee2024gta,li2024augmenting,liu2024enhancing}, model-based approaches \cite{char2022bats,lei2024mgda,zhoufree,hepburn2022modelbased}, and clustering \cite{ghugareclosing2024}. Supervised methods, such as sequence modeling \cite{chen2021decision,huang2025drdt3,wu2024elastic,zhuangreinformer,wang2024critic}, latent space learning \cite{zeng2023goal,venkatraman2023reasoning,zhang2024context}, and dynamic programming \cite{gao2024act,kim2025adaptive}, also demonstrate stitching capabilities. 
We propose a method to solve the stitching problem directly within the sampling process of a compositional generative framework.

\paragraph{Generative trajectory stitching.} Recent works leverage compositional generative modeling \cite{du2020compositional,du2023reduce,garipov2023compositional, du2024compositional,mahajan2024compositional,bradley2025mechanisms,okawa2024compositional,thornton2025composition} for trajectory stitching in planning \cite{yang2023compositional,ajay2023compositional,luo2024potential, mishra2023generative, mishra2024generative, luo2025compdiffuser, mishra2025compositional} and policy learning \cite{wang2024poco,patil2024composing}. We identify that current stitching heuristics ignore how interior predictions shift with \edited{boundary-condition} perturbations, inducing a non-conservative score field with non-zero curl. \edited{By restoring conservativity for the induced bridge-energy model, we introduce a boundary reaction term that propagates interior inconsistencies to chunk boundaries, promoting coherent trajectories.}

\edited{\paragraph{Training-time trajectory stitching.}
Recent methods, such as SCoTS~\citep{lee2025scots} and HM-Diffusion~\citep{chen2025hmdiffusion}, propose addressing long-horizon planning by modifying training data (e.g., via trajectory segmenting and progressive augmentation). While effective, these approaches often require task-specific state representations, limiting transfer to domains with high-dimensional observations or complex dynamics like AntSoccer or manipulation. In contrast, \ourmethod{} is an inference-time composition rule for a fixed local chunk denoiser. Moreover, as analyzed in Appendix~\ref{app:sufficient_conditions_exact_better}, given a well-trained local expert, \ourmethod{} can remove the structural inference-time error associated with the dropped reaction term in the induced bridge-energy model. Thus, ECD is natural to combine with methods that improve the local planner itself, such as SCoTS.}

\section{Conclusion}

\edited{In this paper, we proposed an energy-based compositional diffuser (\ourmethod{}), which composes local energies into a global scalar objective whose negative gradient defines an energy-derived denoising correction. This perspective yields a conservative field for the induced bridge-energy model and exposes a missing boundary reaction term that propagates interior inconsistency back to chunk boundaries. To ensure practical inference, we further introduce a Markov-structured approximation that computes the reaction via a single block-tridiagonal solve per chunk, yielding linear time complexity in the planning horizon. Finally, we validate our approach in maze environments across both low- and high-dimensional settings. Limitations and future directions are discussed in Appendix~\ref{app:diss}.}

\section*{Impact Statement}

\edited{This work improves long-horizon robotic planning by making compositional diffusion planners more consistent and efficient. Potential positive impacts include more reliable robot navigation, manipulation, and task execution learned from fragmented offline data, which may reduce the need for costly full-task demonstrations. However, improved planning methods could also be applied in safety-critical or harmful robotic systems if deployed without appropriate safeguards. Our experiments are limited to simulated benchmark environments, and real-world deployment would require careful validation, monitoring, and oversight to ensure safe and responsible behavior.}

\section*{Acknowledgment}

\edited{Tao Sun is supported by the Stanford Graduate Fellowship (SGF). The authors thank Liyuan Zhu for discussion. We thank Stanford Marlowe Cluster~\cite{kapfer2025marlowe} and NVIDIA Academic Grant Project for providing GPU resources for this project.}

\bibliography{example_paper}
\bibliographystyle{icml2026}

\newpage
\appendix
\onecolumn

\section{\edited{Why Is CompDiffuser's Stitching Non-Conservative?}}
\label{app:cd_nonconservative_full}

\edited{This section proves Proposition~\ref{prop:nonconservative} and makes explicit how the argument applies to the CompDiffuser (CD) variants used in practice. We use the selector notation from the main text for boundary-conditioned chunk denoisers.}

\subsection{Notation}
Let $\mathbf{x}\in\mathbb{R}^{(L+1)D}$ be the long noisy trajectory at diffusion time $t$.  Chunk $k$ selects local coordinates using $\mathbf{S}_k$.  Within the chunk, $\mathbf{P}_k$ selects the coordinates whose noisy mean is predicted or scored, and $\mathbf{O}_k$ selects the coordinates used as the local \edited{boundary condition}.  Thus
\begin{equation}
    y_k = \mathbf{P}_k\mathbf{S}_k\mathbf{x},
    \qquad
    c_k = \mathbf{O}_k\mathbf{S}_k\mathbf{x}.
\end{equation}
\edited{The local denoiser induces a noisy mean}
\edited{\begin{equation}
    \boldsymbol{\mu}_k(c_k,t)
    =
    \sqrt{\bar\alpha_t}\,\hat{\mathbf{x}}_{0,\theta}^{k}(\operatorname{sg}[y_k],c_k,t),
\end{equation}}
\edited{where the dependence on $c_k$ is essential and $\operatorname{sg}[y_k]$ denotes that the selected-coordinate input is treated as stop-gradient for the correction energy. Here $c_k$ is the local boundary condition vector supplied to the chunk denoiser.}

{Let $\Pi$ denote the stitching or aggregation operator used by CD, \textit{e.g.}, overwrite, averaging, overlap blending, or candidate selection followed by blending.}  A stitched noisy mean can be written abstractly as
\begin{equation}
    \boldsymbol{\nu}_{\mathrm{CD}}(\mathbf{x},t)
    =
    \Pi\big(\{\boldsymbol{\mu}_k(\mathbf{O}_k\mathbf{S}_k\mathbf{x},t)\}_{k=0}^{K}\big),
\end{equation}
which induces the usual VP noisy-mean score estimate
\begin{equation}
    s_{\mathrm{CD}}(\mathbf{x},t)
    :=
    \frac{\boldsymbol{\nu}_{\mathrm{CD}}(\mathbf{x},t)-\mathbf{x}}{\sigma_t^2}.
    \label{eq:app_cd_score}
\end{equation}
{For $s_{\mathrm{CD}}$ to be the gradient of a scalar potential, its Jacobian must be symmetric on any simply connected region where the stitching rule is smooth.}

\subsection{\texorpdfstring{{Proof of the Non-Conservative Stitching Proposition}}{Proof of the Non-Conservative Stitching Proposition}}
\label{app:prop31_sync}

\begin{proposition}[Non-conservative score field of chunk stitching]
\label{prop:app_nonconservative_stitching}
Fix $t$ and consider a region where the stitching rule $\Pi$ is differentiable and its selected/blending weights are fixed.  Suppose there exist a local prediction coordinate $a$ and a \edited{boundary-conditioning} coordinate $b\in \mathbf{O}_k\mathbf{S}_k\mathbf{x}$ such that the stitched mean uses $\mu_{k,a}$ with nonzero coefficient and
\begin{equation}
    \frac{\partial \mu_{k,a}}{\partial x_b}\neq 0.
    \label{eq:app_nonzero_condition_derivative}
\end{equation}
\edited{Assume moreover that the remaining stitched terms do not exactly cancel this contribution, so $\partial \nu_a/\partial x_b\neq 0$ on the considered region.} Suppose further that the stitching rule assigns coordinate $b$ by copying, overwriting, averaging, or blending values whose local dependence does not include $x_a$.  Then the Jacobian of $s_{\mathrm{CD}}(\mathbf{x},t)$ is asymmetric, and therefore $s_{\mathrm{CD}}$ is generally non-conservative.
\end{proposition}

\begin{proof}
From Eq.~\eqref{eq:app_cd_score},
\begin{equation}
    \nabla_{\mathbf{x}}s_{\mathrm{CD}}(\mathbf{x},t)
    =
    \sigma_t^{-2}\big(\nabla_{\mathbf{x}}\boldsymbol{\nu}_{\mathrm{CD}}(\mathbf{x},t)-I\big).
\end{equation}
Thus it is enough to show that $\nabla_{\mathbf{x}}\boldsymbol{\nu}_{\mathrm{CD}}$ is asymmetric.  By assumption, on the considered region the stitched coordinate $\nu_a$ contains the local prediction $\mu_{k,a}$ with some nonzero coefficient $\gamma$:
\begin{equation}
    \nu_a(\mathbf{x},t)
    =
    \gamma\,\mu_{k,a}(\mathbf{O}_k\mathbf{S}_k\mathbf{x},t)+h_a(\mathbf{x},t),
    \qquad \gamma\neq 0,
\end{equation}
where $h_a$ collects the remaining stitched terms.  \edited{By the non-cancellation assumption,}
\begin{equation}
    \frac{\partial (s_{\mathrm{CD}})_a}{\partial x_b}
    =
    \sigma_t^{-2}\frac{\partial \nu_a}{\partial x_b}
    \neq 0.
\end{equation}
However, coordinate $b$ is assigned by a heuristic stitching operation that does not reciprocally depend on $x_a$.  Hence, locally,
\begin{equation}
    \frac{\partial (s_{\mathrm{CD}})_b}{\partial x_a}
    =
    \sigma_t^{-2}\frac{\partial \nu_b}{\partial x_a}
    =0.
\end{equation}
{The mixed partial derivatives are unequal, so the Jacobian is asymmetric. Therefore the field cannot be written as $\nabla \phi(\mathbf{x})$ for any scalar potential $\phi$ on that region.}
\end{proof}

\paragraph{Why this is the missing reaction.}
The bridge energy used by ECD has the form
\begin{equation}
    E_k(\mathbf{x},t)
    =
    \frac{1}{2\sigma_t^2}
    \left\|\mathbf{P}_k\mathbf{S}_k\mathbf{x}
    -\boldsymbol{\mu}_k(\mathbf{O}_k\mathbf{S}_k\mathbf{x},t)\right\|_{W_k}^{2}.
\end{equation}
Let
\begin{equation}
    r_k
    =
    W_k\left(\mathbf{P}_k\mathbf{S}_k\mathbf{x}
    -\boldsymbol{\mu}_k(\mathbf{O}_k\mathbf{S}_k\mathbf{x},t)\right),
    \qquad
    J_{O,k}
    =
    \frac{\partial \boldsymbol{\mu}_k}{\partial (\mathbf{O}_k\mathbf{S}_k\mathbf{x})}.
\end{equation}
Then the chain rule gives
\begin{equation}
    -\nabla_{\mathbf{x}}E_k
    =
    -\sigma_t^{-2}\mathbf{S}_k^\top\mathbf{P}_k^\top r_k
    +
    \sigma_t^{-2}\mathbf{S}_k^\top\mathbf{O}_k^\top J_{O,k}^\top r_k.
    \label{eq:app_exact_reaction}
\end{equation}
\edited{The second term is the boundary reaction. Heuristic stitching uses local predictions while omitting this reciprocal derivative, and Proposition~\ref{prop:app_nonconservative_stitching} formalizes the resulting asymmetry.}

\paragraph{Auto-regressive variant.}
\label{app:ar_cd_analysis}
{CD's auto-regressive (AR) variant replaces averaging with directional overwriting, which makes the asymmetry direct.}

\begin{proposition}[Autoregressive overwrite is non-conservative]
\label{prop:app_ar_nonconservative}
\edited{Assume a chunk prediction depends nontrivially on a boundary-condition coordinate. Then both forward and backward autoregressive CD induce a generally non-conservative score field.}
\end{proposition}

\begin{proof}
\edited{Consider two neighboring chunks $k$ and $k+1$ with shared coordinate $a\in C_k\cap C_{k+1}$. In forward AR, the later chunk overwrites the earlier one, so on the shared region}
\begin{equation}
    \nu_a=\mu_{k+1,a}(c_{k+1},t).
\end{equation}
\edited{Let $b$ be a left boundary-condition coordinate of chunk $k+1$. By assumption,}
\begin{equation}
    \frac{\partial \nu_a}{\partial x_b}
    =
    \frac{\partial \mu_{k+1,a}}{\partial x_b}
    \neq 0.
\end{equation}
But $b$ is supplied by the previous chunk or by a fixed boundary condition and is not assigned through a rule that depends on $x_a$.  Thus $\partial \nu_b/\partial x_a=0$, and the Jacobian is asymmetric.

\edited{Backward AR gives the same result with the roles reversed. The earlier chunk overwrites the later one. Choose $a$ in the shared region assigned from chunk $k$ and let $b$ be a right boundary-condition coordinate of chunk $k$. Again $\partial \nu_a/\partial x_b\neq 0$ while $\partial \nu_b/\partial x_a=0$. Hence the score field is generally non-conservative.}
\end{proof}

\paragraph{Chunk-interleaved variant.}
\label{app:interleave_cd_analysis}
{CD's chunk-interleaved variant used in practice is an ordered composition of local updates. It is therefore better viewed as a proposal operator than as a single simultaneous score field. Its non-conservative character can still be formalized.}
\edited{Let $\mathcal{T}_k$ denote the local update that denoises chunk $k$ while treating the current neighboring boundary variables as conditions.}  One interleaved sweep is
\begin{equation}
    \mathcal{T}_{\mathrm{int}}
    =
    \mathcal{T}_K\circ\cdots\circ\mathcal{T}_1\circ\mathcal{T}_0.
\end{equation}
{If a sweep were an Euler step of a conservative score field, its infinitesimal displacement would have a symmetric Jacobian. The following result shows why this condition fails.}

\begin{proposition}[Interleaving does not restore integrability]
\label{prop:app_interleave_nonintegrable}
\edited{Suppose at least one local chunk update has a nonzero dependence on a boundary condition supplied by another chunk, and the reciprocal derivative is omitted as in Proposition~\ref{prop:app_nonconservative_stitching}.}  Then the first-order displacement induced by an interleaved sweep has an asymmetric Jacobian.  Moreover, even if individual block directions were conservative, the ordered composition generically introduces non-gradient Lie-bracket terms unless the block Jacobians commute.
\end{proposition}

\begin{proof}
Write a small local update as $\mathcal{T}_k(\mathbf{x})=\mathbf{x}+\eta f_k(\mathbf{x})+O(\eta^2)$.  The first-order displacement of one sweep is
\begin{equation}
    \mathcal{T}_{\mathrm{int}}(\mathbf{x})-\mathbf{x}
    =
    \eta\sum_k f_k(\mathbf{x})+O(\eta^2).
\end{equation}
\edited{Each $f_k$ uses the local denoiser response to the current boundary condition while treating the boundary-condition source as externally provided. This drops the reciprocal reaction derivative. By Proposition~\ref{prop:app_nonconservative_stitching}, the Jacobian of $\sum_k f_k$ is generically asymmetric.}

At second order, the ordered composition contributes terms of the form
\begin{equation}
    \eta^2\sum_{j>i} Df_j(\mathbf{x})f_i(\mathbf{x}).
\end{equation}
Equivalently, the modified vector field contains Lie brackets $[f_j,f_i]=Df_i f_j-Df_j f_i$.  Lie brackets of gradient fields are not generally gradient fields unless the corresponding Hessians commute in a special way.  Thus interleaving is a useful proposal mechanism, but it does not restore the existence of a global scalar energy.
\end{proof}

\paragraph{Conclusion.}
\edited{CD remains useful as a local proposal because it preserves chunk-level multimodality. However, its update is not the gradient of a global compatibility energy because it treats boundary conditions as one-way inputs. ECD keeps the local proposal and adds the missing reaction correction analyzed next.}

\section{\edited{Why ECD Helps with CompDiffuser?}}
\label{app:exact_ecd_effectiveness}

\edited{The practical sampler uses CD's chunk-interleaved proposal, so the full sampler is not a pure conservative reverse-diffusion field. The theoretical claim is therefore conditional: given a CD proposal, the ECD step is a conservative energy correction that reduces unnecessary independence across chunks, and is closer to the oracle score under certain conditions.}

\subsection{Conservative Correction}
\label{app:proposal_plus_correction}

Let $\mathbf{z}$ denote the duplicated chunk state used by the interleaved sampler.  A CD sweep produces a proposal
\begin{equation}
    \widetilde{\mathbf{z}}
    =
    D_{\theta,t}^{\mathrm{CD}}(\mathbf{z}).
\end{equation}
ECD then applies an energy correction using the lifted bridge energy
\begin{equation}
    \mathcal{E}_t(\mathbf{z})
    =
    \sum_k \frac{1}{2\sigma_t^2}
    \left\|z^k-\boldsymbol{\mu}_k(O_k\mathbf{z},t)\right\|_{W_k}^{2}.
    \label{eq:app_lifted_energy}
\end{equation}
The exact correction is
\begin{equation}
    \mathbf{z}^{+}
    =
    \widetilde{\mathbf{z}}-
    \rho_t\,\Pi_c\big(\nabla \mathcal{E}_t(\widetilde{\mathbf{z}})\big),
    \label{eq:app_exact_correction}
\end{equation}
{where $\Pi_c$ is optional blockwise clipping. In the implementation, the bridge and reaction terms may be scaled separately; the proposition below states the clean case, and the same proof gives an alignment condition for unequal scales.}

\begin{proposition}[Exact ECD is an energy-improving correction to CD]
\label{prop:app_exact_energy_descent}
Assume $\mathcal{E}_t$ has $L_t$-Lipschitz gradient in a neighborhood of the CD proposal $\widetilde{\mathbf{z}}$.  If $\Pi_c$ is the identity and $0<\rho_t<2/L_t$, then the exact ECD correction in Eq.~\eqref{eq:app_exact_correction} satisfies
\begin{equation}
    \mathcal{E}_t(\mathbf{z}^{+})
    \le
    \mathcal{E}_t(\widetilde{\mathbf{z}})
    -
    \rho_t\left(1-\frac{L_t\rho_t}{2}\right)
    \left\|\nabla \mathcal{E}_t(\widetilde{\mathbf{z}})\right\|^2.
    \label{eq:app_exact_energy_descent}
\end{equation}
With clipping, the same conclusion holds whenever
\begin{equation}
    \left\langle \nabla\mathcal{E}_t(\widetilde{\mathbf{z}}),\Pi_c(\nabla\mathcal{E}_t(\widetilde{\mathbf{z}}))\right\rangle
    >
    \frac{L_t\rho_t}{2}
    \left\|\Pi_c(\nabla\mathcal{E}_t(\widetilde{\mathbf{z}}))\right\|^2.
\end{equation}
\end{proposition}

\begin{proof}
{The smoothness inequality gives}
\begin{equation}
    E(u-\rho g)\le E(u)-\rho\|g\|^2+\frac{L\rho^2}{2}\|g\|^2
\end{equation}
for $g=\nabla E(u)$.  Setting $u=\widetilde{\mathbf{z}}$ yields Eq.~\eqref{eq:app_exact_energy_descent}.  The clipped statement follows from the same smoothness inequality with direction $-\Pi_c(g)$.
\end{proof}

\subsection{\edited{Sufficient Conditions for ECD to Improve CD}}
\label{app:sufficient_conditions_exact_better}

{We further give two sufficient conditions. The first is energy-based and applies directly to the post-CD correction. The second is score-based and connects the correction to reverse-diffusion convergence.}

\begin{proposition}[Sublevel success improvement]
\label{prop:app_sublevel_success}
Let $G_\tau=\{\mathbf{z}:\mathcal{E}_t(\mathbf{z})\le \tau\}$ be a set of energy-consistent proposals.  Suppose the exact ECD correction satisfies
\begin{equation}
    \mathcal{E}_t(C(\widetilde{\mathbf{z}}))\le \mathcal{E}_t(\widetilde{\mathbf{z}})
\end{equation}
for every CD proposal $\widetilde{\mathbf{z}}$, where $C$ denotes the correction map.  Then
\begin{equation}
    \mathbb{P}\big(C(\widetilde{\mathbf{Z}})\in G_\tau\big)
    \ge
    \mathbb{P}\big(\widetilde{\mathbf{Z}}\in G_\tau\big).
\end{equation}
If, in addition, a set of positive probability satisfies
\begin{equation}
    \tau<\mathcal{E}_t(\widetilde{\mathbf{Z}})\le \tau+\Delta
    \quad\text{and}\quad
    \mathcal{E}_t(\widetilde{\mathbf{Z}})-\mathcal{E}_t(C(\widetilde{\mathbf{Z}}))\ge \Delta,
\end{equation}
then the inequality is strict.
\end{proposition}

\begin{proof}
If $\widetilde{\mathbf{z}}\in G_\tau$, then energy monotonicity implies $C(\widetilde{\mathbf{z}})\in G_\tau$.  Thus $C^{-1}(G_\tau)$ contains $G_\tau$, proving the first inequality.  The strict statement follows because the specified positive-probability shell outside $G_\tau$ is mapped into $G_\tau$.
\end{proof}

{This result does not assume that low energy is exactly equivalent to task success. It guarantees that ECD improves the probability of satisfying the energy-consistency certificate. When the bridge energy is a useful feasibility surrogate, this certificate becomes a mechanism for improving success rate.}

For the score-based condition, let $s_\star$ be the oracle score of the induced bridge-energy model.  Let $s_{\mathrm{ex}}$ be the exact ECD score/correction direction and $s_{\mathrm{CD}}$ be the CD direction using the same local denoiser but without the reaction.  Define
\begin{equation}
    R:=s_{\mathrm{ex}}-s_{\mathrm{CD}},
    \qquad
    e:=s_{\mathrm{ex}}-s_\star.
\end{equation}
Then $s_{\mathrm{CD}}-s_\star=e-R$.

\begin{theorem}[Score-error condition for exact ECD to outperform CD]
\label{thm:app_score_condition_exact_better}
At a fixed diffusion time and state, exact ECD has smaller squared score error than CD if and only if
\begin{equation}
    \|e\|^2<\|e-R\|^2
    \quad\Longleftrightarrow\quad
    \langle e,R\rangle<\frac{1}{2}\|R\|^2.
    \label{eq:app_score_condition_exact_better}
\end{equation}
A sufficient condition is $\|e\|<\frac{1}{2}\|R\|$.
\end{theorem}

\begin{proof}
Compute
\begin{equation}
    \|s_{\mathrm{CD}}-s_\star\|^2-\|s_{\mathrm{ex}}-s_\star\|^2
    =
    \|e-R\|^2-\|e\|^2
    =
    \|R\|^2-2\langle e,R\rangle.
\end{equation}
The stated equivalence and sufficient condition follow immediately.
\end{proof}

\paragraph{Connection to diffusion convergence.}
{Standard reverse-diffusion error bounds depend on a weighted integral or sum of squared score errors. If Eq.~\eqref{eq:app_score_condition_exact_better} holds on average along the reverse process, exact ECD has a smaller effective score-error term than CD for the induced bridge model. This statement does not identify the bridge model with the true data distribution. It states that, for the energy model defined by the local denoisers, CD carries an additional structural error from the dropped reaction, while exact ECD removes that error up to learned-model error.}

\paragraph{Oracle sufficient condition.}
The condition above can be related to local denoiser accuracy.  Let $\mu_k^\star,J_{O,k}^\star,r_k^\star$ denote the oracle local mean, \edited{boundary Jacobian}, and residual.  Let $m_k^\star=(J_{O,k}^\star)^\top r_k^\star$ and $m_k=J_{O,k}^\top r_k$.  If
\begin{equation}
\sum_k\left[(1+\|J_{O,k}\|)\|W_k\|\,\|\mu_k-\mu_k^\star\|
+\|J_{O,k}-J_{O,k}^\star\|\,\|r_k^\star\|\right]
<
\frac{1}{2}\left\|\sum_k O_k^\top m_k\right\|,
\label{eq:app_oracle_sufficient_condition}
\end{equation}
\edited{then exact ECD is closer to the oracle noisy-mean score than CD. This formalizes the intuition that exact ECD helps when the local chunk mean and its boundary-condition dependence are accurate enough relative to the reaction term that CD drops.}

\section{\texorpdfstring{\edited{Markov Approximation to the Boundary Reaction Term}}{Markov Approximation to the Boundary Reaction Term}}
\label{app:approx_close_exact}

{Exact ECD requires the neural VJP $J_{O,k}^\top r_k$. Approximate ECD replaces it with a Markov reaction computed by a sparse linear solve. This section analyzes the approximation error and clarifies when the approximation should be trusted.}

\subsection{Implicit Markov Reaction}
\label{app:implicit_markov_reaction}

For chunk $k$, introduce a local Markov surrogate energy
\begin{equation}
    E_{k}^{\mathrm{ms}}(y,c,t)
    =
    \sum_i E_{\phi,t}(y_i)
    +
    \sum_i E_{\psi,t}(y_i,y_{i+1})
    +
    E_{\partial,t}(y,c),
    \label{eq:app_markov_surrogate}
\end{equation}
where $y$ denotes predicted/core chunk coordinates and $c$ denotes \edited{boundary conditions}.  Let
\begin{equation}
    H_k(t):=\nabla_{yy}^{2}E_k^{\mathrm{ms}}(y,c,t),
    \qquad
    G_k(t):=\nabla_{yc}^{2}E_k^{\mathrm{ms}}(y,c,t).
\end{equation}
If the surrogate mean is defined by the stationarity condition $\nabla_yE_k^{\mathrm{ms}}(y,c,t)=0$, implicit differentiation gives
\begin{equation}
    \frac{\partial y}{\partial c}
    =
    -H_k(t)^{-1}G_k(t).
\end{equation}
Therefore the exact reaction message is approximated by
\begin{equation}
    \widetilde m_k(t)=-G_k(t)^\top u_k(t),
    \qquad
    H_k(t)u_k(t)=r_k(t).
    \label{eq:app_markov_message}
\end{equation}
{When $E_\psi$ is first-order Markov, $H_k$ is block-tridiagonal, so the solve is linear in the chunk horizon for fixed state dimension.}

\subsection{Single-Step Reaction Error}
\label{app:single_step_reaction_error}

\begin{theorem}[Markov reaction approximation error]
\label{thm:app_markov_reaction_error}
Let $H_k^\star,G_k^\star$ be the ideal implicit matrices satisfying $J_{O,k}^\star=-(H_k^\star)^{-1}G_k^\star$.  Assume
\begin{equation}
    \lambda_{\min}(H_k^\star)\ge m_k^\star(t)>0,
    \qquad
    \|H_k-H_k^\star\|\le \delta_H^k(t)<m_k^\star(t),
    \qquad
    \|G_k-G_k^\star\|\le \delta_G^k(t).
\end{equation}
Let $m_k=J_{O,k}^{\star\top}r_k$ and $\widetilde m_k=-G_k^\top H_k^{-1}r_k$.  Then
\begin{equation}
\|\widetilde m_k-m_k\|
\le
\left[
\frac{\delta_G^k(t)}{m_k^\star(t)}
+
\frac{\delta_H^k(t)\|G_k(t)\|}{m_k^\star(t)(m_k^\star(t)-\delta_H^k(t))}
\right]
\|r_k(t)\|.
\label{eq:app_markov_msg_error_no_sigma}
\end{equation}
If the bridge energy has VP scaling $\lambda(2\sigma_t^2)^{-1}\|z^k-\mu_k\|_{W_k}^{2}$ and $0<w_{\min}^k\le w_j^k\le w_{\max}^k$, then
\begin{equation}
\begin{aligned}
\|\widetilde m_k-m_k\|
&\le
\sigma_t\frac{w_{\max}^k}{\sqrt{w_{\min}^k}}
\sqrt{\frac{2}{\lambda}}
\left[
\frac{\delta_G^k(t)}{m_k^\star(t)}
+
\frac{\delta_H^k(t)\|G_k(t)\|}{m_k^\star(t)(m_k^\star(t)-\delta_H^k(t))}
\right]
\sqrt{E_{\mathrm{br}}^k(\mathbf{z},t)}.
\end{aligned}
\label{eq:app_markov_msg_error_sigma}
\end{equation}
\end{theorem}

\begin{proof}
Write
\begin{equation}
    m_k=-(G_k^\star)^\top(H_k^\star)^{-1}r_k,
    \qquad
    \widetilde m_k=-G_k^\top H_k^{-1}r_k.
\end{equation}
Add and subtract $-G_k^\top(H_k^\star)^{-1}r_k$.  The $G$-error term is bounded by
\begin{equation}
    \delta_G^k(t)\|(H_k^\star)^{-1}\|\|r_k\|
    \le
    \frac{\delta_G^k(t)}{m_k^\star(t)}\|r_k\|.
\end{equation}
For the inverse error, use the resolvent identity
\begin{equation}
    H_k^{-1}-(H_k^\star)^{-1}
    =
    H_k^{-1}(H_k^\star-H_k)(H_k^\star)^{-1}
\end{equation}
and
\begin{equation}
    \|H_k^{-1}\|
    \le
    \frac{1}{m_k^\star(t)-\delta_H^k(t)}.
\end{equation}
Combining these bounds gives Eq.~\eqref{eq:app_markov_msg_error_no_sigma}.  The VP-scaled statement follows from
\begin{equation}
    \|r_k(t)\|
    \le
    \frac{w_{\max}^k}{\sqrt{w_{\min}^k}}
    \sqrt{\frac{2\sigma_t^2}{\lambda}}
    \sqrt{E_{\mathrm{br}}^k(\mathbf{z},t)}.
\end{equation}
\end{proof}

\paragraph{Implication.} \edited{The key implication is that the approximation gap shrinks with $\sigma_t$ along bounded-energy denoising paths. Thus the Markov reaction is most reliable in the low-noise regime, where local modes and boundary dependencies are sharpest.}

\subsection{Approximate ECD versus Exact ECD and CD}
\label{app:approx_vs_exact_cd}

Let
\begin{equation}
    R_{\mathrm{ex}}=\sum_k O_k^\top m_k,
    \qquad
    R_{\mathrm{ap}}=\sum_k O_k^\top \widetilde m_k.
\end{equation}
Then
\begin{equation}
    \|R_{\mathrm{ap}}-R_{\mathrm{ex}}\|
    \le
    \sum_k\|O_k\|\,\|\widetilde m_k-m_k\|.
    \label{eq:app_reaction_total_gap}
\end{equation}
By contrast, bridge-only/CD drop $R_{\mathrm{ex}}$ entirely.  Hence approximate ECD is closer to exact ECD than bridge-only whenever
\begin{equation}
    \sum_k\|O_k\|\,\|\widetilde m_k-m_k\|
    <
    \|R_{\mathrm{ex}}\|.
    \label{eq:app_approx_beats_bridge_condition}
\end{equation}
{This gives a useful diagnostic: approximate ECD is preferable once its Markov error is smaller than the reaction that CD omits.}

\begin{proposition}[Approximate descent with relative reaction error]
\label{prop:app_approx_descent}
\edited{Let $g=\nabla\mathcal{E}_t(\widetilde{\mathbf{z}})$ be the exact ECD gradient, let $\widehat g$ be the approximate ECD gradient, and define $p=\Pi_c(\widehat g)$. Suppose the clipping map is alignment-preserving,}
\edited{\begin{equation}
    \langle \widehat g,p\rangle \ge \|p\|^2,
    \label{eq:clip_alignment}
\end{equation}}
\edited{which holds for standard norm clipping and coordinatewise clipping. Suppose further that}
\edited{\begin{equation}
    \|\widehat g-g\|\le \kappa\|p\|,
    \qquad \kappa<1,
\end{equation}}
\edited{and $\nabla\mathcal{E}_t$ is $L_t$-Lipschitz. Then the update $\mathbf{z}^+=\widetilde{\mathbf{z}}-\eta p$ satisfies}
\edited{\begin{equation}
    \mathcal{E}_t(\mathbf{z}^+)
    \le
    \mathcal{E}_t(\widetilde{\mathbf{z}})
    -
    \eta\left(1-\kappa-\frac{L_t\eta}{2}\right)
    \|p\|^2.
\end{equation}}
\edited{Therefore approximate ECD is an energy-decreasing correction whenever $\eta<2(1-\kappa)/L_t$. Under Theorem~\ref{thm:app_markov_reaction_error}, $\kappa$ decreases with $\sigma_t$ when the bridge energy and curvature mismatch remain controlled.}
\end{proposition}

\begin{proof}
\edited{By smoothness,}
\edited{\begin{equation}
\mathcal{E}_t(\widetilde{\mathbf{z}}-\eta p)
\le
\mathcal{E}_t(\widetilde{\mathbf{z}})
-\eta\langle g,p\rangle
+\frac{L_t\eta^2}{2}\|p\|^2.
\end{equation}}
\edited{Write $g=\widehat g-(\widehat g-g)$. Using Eq.~\eqref{eq:clip_alignment} and Cauchy's inequality gives}
\edited{\begin{equation}
\langle g,p\rangle
\ge
\|p\|^2-\|\widehat g-g\|\|p\|
\ge
(1-\kappa)\|p\|^2.
\end{equation}}
\edited{Substituting this bound into the smoothness inequality yields the result.}
\end{proof}

\subsection{Multi-Step Perturbation from Exact ECD}
\label{app:reaction_error_bound}

Let $\Phi_i^{\mathrm{ex}}$ and $\Phi_i^{\mathrm{ap}}$ be one exact and approximate ECD update at DDIM step $i$.  Suppose $\Phi_i^{\mathrm{ex}}$ is $M_i$-Lipschitz and
\begin{equation}
    \|\Phi_i^{\mathrm{ap}}(\mathbf{z})-\Phi_i^{\mathrm{ex}}(\mathbf{z})\|
    \le
    \eta_i\epsilon_i(t_i).
\end{equation}
If both samplers start from the same initial noise, then after $N$ steps,
\begin{equation}
    \|\mathbf{z}_0^{\mathrm{ap}}-\mathbf{z}_0^{\mathrm{ex}}\|
    \le
    \sum_{i=1}^{N}\left(\prod_{j<i}M_j\right)\eta_i\epsilon_i(t_i).
    \label{eq:app_multistep_approx_exact}
\end{equation}
If $M_j\le 1+L_j\eta_j$, then
\begin{equation}
    \|\mathbf{z}_0^{\mathrm{ap}}-\mathbf{z}_0^{\mathrm{ex}}\|
    \le
    \sum_i \eta_i\epsilon_i(t_i)\exp\left(\sum_{j<i}L_j\eta_j\right).
\end{equation}
{Combining this with Eq.~\eqref{eq:app_markov_msg_error_sigma} gives $\epsilon_i(t_i)=O(\sigma_{t_i}\sqrt{E_{\mathrm{br}}(\mathbf{z}_i,t_i)})$ under the Markov assumptions. This motivates late-time reaction gates: high-noise corrections are weakly trusted, while low-noise corrections are both more informative and better approximated.}

\section{Experimental Details}
\label{app:additional_details}

{This appendix describes the implementation choices used in our OGBench stitching experiments. We follow the CompDiffuser (CD) evaluation protocol whenever possible: all methods use the same local planner checkpoints, candidate batch size, DDIM grid, boundary-mismatch ranking, and trajectory-blending post-processing, and replanning rule. \edited{Note that the main difference between CD and ECD is the additional bridge/boundary-reaction correction applied after each CD-style local denoising proposal.}}

\subsection{OGBench Evaluation Protocol}
\label{app:ogbench_config_table}

{Table~\ref{tab:ogbench_config} summarizes the evaluation settings used across the 13 OGBench stitching environments. The ``training horizon'' column is the local chunk horizon $l$ used by the planner. The planning horizon $L$ is the total stitched horizon induced by the default number of chunks,}
\begin{equation}
    L = l + K(l-o),
\end{equation}
{where $o=l-d$ is the overlap length and $K+1$ is the number of chunks. The 15D AntMaze variants use the same underlying AntMaze environments but plan in the first 15 observation dimensions.}
{
\begin{table*}[h]
\centering
\makeblue
\footnotesize
\setlength{\tabcolsep}{2.7pt}
\begin{tabular}{l c c c c c c}
\toprule
\multirow{2}{*}{\textbf{OGBench environment}} & \textbf{Training horizon} & \textbf{Planning horizon} & \textbf{Overlap} & \textbf{\# Chunks} & \textbf{Max steps} & \textbf{Replan thres.} \\
 & $l$ & $L$ & $o=l-d$ & $K+1$ & $N_{\max}$ & $\tau$ \\
\midrule
\texttt{pointmaze-medium-stitch-v0} & 160 & 368 & 56 & 3 & 1000 & 1 \\
\texttt{pointmaze-large-stitch-v0} & 160 & 680 & 56 & 6 & 1000 & 1 \\
\texttt{pointmaze-giant-stitch-v0} & 160 & 888 & 56 & 8 & 1000 & 1 \\
\midrule
\texttt{antmaze-medium-stitch-v0} & 160 & 368 & 56 & 3 & 1000 & 4 \\
\texttt{antmaze-large-stitch-v0} & 160 & 680 & 56 & 6 & 1000 & 4 \\
\texttt{antmaze-giant-stitch-v0} & 160 & 992 & 56 & 9 & 2000 & 4 \\
\midrule
\texttt{humanoidmaze-medium-stitch-v0} & 336 & 960 & 128 & 4 & 5000 & 10 \\
\texttt{humanoidmaze-large-stitch-v0} & 336 & 1376 & 128 & 6 & 5000 & 10 \\
\texttt{humanoidmaze-giant-stitch-v0} & 336 & 2416 & 128 & 11 & 8000 & 10 \\
\midrule
\texttt{antsoccer-arena-stitch-v0} & 160 & 576 & 56 & 5 & 5000 & 4 \\
\texttt{antsoccer-medium-stitch-v0} & 160 & 680 & 56 & 6 & 5000 & 6 \\
\midrule
\edited{\texttt{antmaze-large-stitch-v0-o15d}} & \edited{160} & \edited{680} & \edited{56} & \edited{6} & \edited{1000} & \edited{4} \\
\edited{\texttt{antmaze-giant-stitch-v0-o15d}} & \edited{160} & \edited{992} & \edited{56} & \edited{9} & \edited{2000} & \edited{4} \\
\bottomrule
\end{tabular}
\vspace{0.5em}
\caption{\textbf{OGBench stitching evaluation configuration.} {The table reports the local planner horizon, total stitched planning horizon, overlap length, default number of chunks, rollout horizon, and adaptive replanning threshold. The o15d AntMaze variants use the same underlying AntMaze environments but plan in the first 15 observation dimensions.}}
\label{tab:ogbench_config}
\end{table*}
}

\subsection{Implementation Details}
\label{app:implementation_details}

\paragraph{Planner models.}
{Planner architectures are determined by the environment family and planner state dimension. \edited{Two-dimensional PointMaze, AntMaze, and HumanoidMaze use a one-dimensional U-Net trajectory denoiser with separate boundary-condition encoders. AntSoccer and AntMaze-o15d use a DiT trajectory denoiser with DiT boundary-condition encoders, which are more stable for high-dimensional planner states. All planners use the same conditioning convention: a local chunk receives boundary conditions from either fixed start/goal inpainting tokens or neighboring chunks.} Table~\ref{tab:planner_network} summarizes the planner architectures.}

\begin{table*}[h]
\centering
\makeblue
\footnotesize
\setlength{\tabcolsep}{4pt}
\begin{tabular}{l c l l}
\toprule
\textbf{Environment family} & \textbf{Planner dim.} & \textbf{Trajectory denoiser} & \textbf{\edited{Boundary encoder}} \\
\midrule
\texttt{pointmaze-*} & 2D
& U-Net, base dim 128, dim mults $(1,2,4,8)$
& \edited{CNN boundary encoders, out dim 256} \\
\texttt{antmaze-*} & 2D
& U-Net, base dim 128, dim mults $(1,2,4,8)$
& \edited{CNN boundary encoders, out dim 256} \\
\texttt{humanoidmaze-*} & 2D
& U-Net, base dim 128, dim mults $(1,2,4,8)$
& \edited{CNN boundary encoders, out dim 256} \\
\texttt{antsoccer-*} & 17D
& DiT, hidden 768, depth 16, 12 heads
& \edited{DiT boundary encoders, hidden 384, depth 8} \\
\texttt{antmaze-*-o15d} & 15D
& DiT, hidden 768, depth 16, 12 heads
& \edited{DiT boundary encoders, hidden 384, depth 8} \\
\bottomrule
\end{tabular}
\vspace{0.5em}
\caption{\textbf{Planner architectures.} {The planner architecture is fixed by environment family and planner dimensionality. The o15d AntMaze variants use the same underlying AntMaze environments but plan in the first 15 observation dimensions. All CD and ECD comparisons use the same planner checkpoint within each environment.}}
\label{tab:planner_network}
\end{table*}

\paragraph{Inverse-dynamics models.}
{PointMaze is executed with the standard PD controller and does not use inverse dynamics. AntMaze, HumanoidMaze, AntSoccer, and AntMaze-o15d use an MLP inverse-dynamics model that maps the current full observation and the next planned waypoint to an action. The inverse-dynamics model is shared across CD and ECD evaluations for a fixed environment and checkpoint, so ECD changes only the planner output. Table~\ref{tab:invdyn_network} summarizes the controller and inverse-dynamics models.}

\begin{table*}[h]
\centering
\makeblue
\footnotesize
\setlength{\tabcolsep}{4pt}
\begin{tabular}{l l l}
\toprule
\textbf{Environment family} & \textbf{Model} & \textbf{{Implementation}} \\
\midrule
\texttt{pointmaze-*}
& PD controller
& Standard PD controller with PD gain $=5.0$. \\
\texttt{antmaze-*} 
& MLP inverse dynamics
& Hidden dims $(512,512,512)$. \\
\texttt{humanoidmaze-*}
& MLP inverse dynamics
& Hidden dims $(512,1024,1024,512,256)$. \\
\texttt{antsoccer-arena-stitch-v0}
& MLP inverse dynamics
& Hidden dims $(512,512,512)$. \\
\texttt{antsoccer-medium-stitch-v0}
& MLP inverse dynamics
& Hidden dims $(512,1024,1024,512)$. \\
\texttt{antmaze-*-o15d}
& MLP inverse dynamics
& AntMaze inverse-dynamics architecture with 15D planned input. \\
\bottomrule
\end{tabular}
\vspace{0.5em}
\caption{\textbf{Controller and inverse-dynamics models.} {PointMaze uses a PD controller. All other environments use learned inverse dynamics. For each environment, CD and ECD use the same inverse-dynamics checkpoint, so the comparison isolates the effect of the planner and compositional denoising rule.}}
\label{tab:invdyn_network}
\end{table*}

\paragraph{DDIM and guidance defaults.}
{Unless otherwise specified, all diffusion-planning methods use $50$ DDIM inference steps, $\eta=1.0$, classifier-free guidance weight $2.0$, and candidate batch size $b=40$ for full rollout evaluation. We keep these values fixed between CD and ECD. The environment-specific number of chunks is listed in Table~\ref{tab:ogbench_config}.}

\begin{algorithm}[H]
\caption{CD-style interleaved DDIM with ECD correction}
\label{alg:impl_ecd_ddim}
\begin{algorithmic}[1]
\STATE Initialize chunk states $z_T^0,\ldots,z_T^K\sim\mathcal N(0,I)$.
\STATE Clamp the start state in chunk $0$ and the goal state in chunk $K$.
\FOR{DDIM timesteps $t=T,T-\Delta,\ldots,0$}
    \FOR{chunks $k=0,\ldots,K$ in interleaved order}
        \STATE \edited{Build the local boundary condition $O_k\mathbf z_t$ from fixed start/goal tokens or neighboring boundary variables.}
        \STATE Run the local CD DDIM step to obtain $\tilde z_{t-\Delta}^k$ and $\hat z_{0,\theta}^k$.
        \STATE \edited{Compute $\mu_k=\sqrt{\bar\alpha_t}\hat z_{0,\theta}^k$ and residual $r_k=W_k(P_k z_t^k-\mu_k)$.}
        \STATE \edited{Bridge correction: $\Delta_k\leftarrow-\eta_b(t)\,P_k^\top\Pi_c(r_k)$.}
        \IF{using approximate ECD}
            \STATE \edited{Solve the Markov system $H_k u_k=r_k$ and compute the boundary-space message $\tilde m_k=-G_k^\top u_k$.}
            \STATE \edited{Reaction correction: $\Delta_k\leftarrow\Delta_k+\eta_r(t)\,O_k^\top\Pi_c(\tilde m_k)$.}
        \ENDIF
        \STATE Update $z_{t-\Delta}^k\leftarrow \tilde z_{t-\Delta}^k+\Delta_k$.
    \ENDFOR
    \STATE Re-clamp the start and goal chunks.
\ENDFOR
\STATE \edited{Rank candidates by boundary mismatch, blend shared boundaries, and execute the selected trajectory.}
\end{algorithmic}
\end{algorithm}

\paragraph{DDIM inference loop.}
{Algorithm~\ref{alg:impl_ecd_ddim} gives the inference loop used in the main ECD experiments. The loop follows CD's interleaved chunk sampler. For each diffusion step and chunk, we first run the local CD DDIM proposal and then add the ECD correction. Setting the bridge and reaction scales to zero recovers CD exactly. The bridge-only variant uses only the residual correction, while the approximate variant additionally computes the Markov reaction.} \edited{The time-dependent scales $\eta_b(t)$ and $\eta_r(t)$ absorb both the $\sigma_t^{-2}$ factor from the energy gradient and the DDIM conversion from a noisy-mean correction to an $x_{t-\Delta}$ update. In Algorithm~\ref{alg:impl_ecd_ddim}, $z_t^k$ is the $k$-th local chunk at diffusion step $t$; $P_k$ selects the scored coordinates; $O_k$ selects the boundary-condition variables; $W_k$ is the diagonal residual weight; $\Pi_c$ denotes clipping with threshold $c$; $H_k$ is the approximate Markov Hessian/precision system; and $G_k$ maps boundary variables to the scored coordinates in the local linearization.}

\paragraph{PyTorch pseudo code.}
\edited{The following code snippet summarizes the core ECD implementation. It is simplified for readability. Here \texttt{Pk}, \texttt{Pk\_T}, and \texttt{Ok\_T} denote the selector and scatter maps corresponding to $P_k$, $P_k^\top$, and $O_k^\top$. Please refer to our codebase for full implementation.}

\begin{lstlisting}[language=Python,basicstyle=\ttfamily\footnotesize,breaklines=true]
# Local CD proposal for chunk k.
x_in, tj_cond = build_interleaved_condition(x_list, k, t)
x_prev, x0_pred = diffusion.ddim_p_sample(
    x_in, tj_cond, t_2d, eta=ddim_eta,
    use_clipped_model_output=True, return_x0=True)

# Bridge residual in noisy-mean coordinates.
# Pk selects scored coordinates; Pk_T scatters them back.
mu_t = sqrt_alpha_t * Pk(x0_pred)
r = bridge_weight_mask(k) * (Pk(x_in) - mu_t)

# ECD correction (bridge term). Note: bridge_scale_t absorbs sigma_t^{-2}.
delta = -bridge_scale_t * Pk_T(clip_vec(r, delta_clip))

# ECD correction (approximate reaction term).
if react_scale_t > 0:
    u = solve_block_tridiagonal(H_k, r)
    m_cond = -G_k.transpose(-1, -2) @ u
    delta = delta + react_scale_t * Ok_T(
        clip_vec(m_cond, delta_clip))

x_list[k] = x_prev + delta
\end{lstlisting}

{
\makeblue

\paragraph{Candidate ranking and trajectory blending.}
\label{app:rank_blend}
{CD uses boundary-mismatch ranking and exponential boundary blending. We follow this post-processing for all main experiments. \edited{Let $z^{k,a}$ denote candidate $a$ for local chunk $k$, with $a=1,\ldots,b$ and $k=0,\ldots,K$. Let $\mathsf{R}_o$ select the right boundary segment, i.e., the last $o$ states of a chunk, and let $\mathsf{L}_o$ select the left boundary segment, i.e., the first $o$ states of a chunk. We compute the boundary-mismatch score}}
\begin{equation}
    \edited{d_a=\sum_{k=0}^{K-1}
    \frac{1}{oD}\left\|\mathsf{R}_o z^{k,a}-\mathsf{L}_o z^{k+1,a}\right\|_F^2,
    \qquad a=1,\ldots,b,}
\end{equation}
\edited{where $o$ is the boundary length, $D$ is the planner state dimension, and $\|\cdot\|_F$ is the Frobenius norm over the selected boundary segment. We keep the top $n_{\mathrm{top}}=5$ candidates with the smallest $d_a$. We use only this boundary-mismatch ranking in the main experiments; additional progress- or inverse-dynamics-based rankers are implementation diagnostics and are not used in the reported main configuration.}

{For each selected candidate, non-boundary regions are copied directly. In boundary regions shared by neighboring chunks, we use the same exponential blending rule as CD. \edited{For two adjacent boundary segments $u_{0:o-1}=\mathsf{R}_o z^{k,a}$ and $v_{0:o-1}=\mathsf{L}_o z^{k+1,a}$, indexed by $i=0,\ldots,o-1$, the blended segment is}}
\begin{equation}
    \mathrm{Blend}(u_i,v_i)=\omega_i u_i+(1-\omega_i)v_i,
    \qquad
    \omega_i=\frac{\exp(-\beta i/(o-1))-\exp(-\beta)}{1-\exp(-\beta)},
\end{equation}
{with $\beta=2.0$ by default. Unless otherwise specified, we use CD's \texttt{pick\_type=first}: after top-$k$ ranking and blending, the first ranked trajectory is executed. \edited{In the pseudo-code below, \texttt{trajs[k]} stores the $b$ candidate trajectories for chunk $k$, \texttt{n\_comp}$=K+1$, \texttt{top\_n}$=n_{\mathrm{top}}$, and \texttt{overlap}$=o$.}}

\begin{lstlisting}[language=Python,basicstyle=\ttfamily\footnotesize,breaklines=true]
# CD-style ranking and blending.
scores = []
for k in range(n_comp - 1):
    end_k = trajs[k][:, -overlap:]
    st_next = trajs[k + 1][:, :overlap]
    scores.append(((end_k - st_next) ** 2).mean(axis=(1, 2)))
rank = np.argsort(np.stack(scores, axis=1).sum(axis=1))
top = [traj[rank[:top_n]] for traj in trajs]
plan = exponential_overlap_blend(top, beta=2.0)
plan = plan[0]  # pick_type='first'
\end{lstlisting}

\paragraph{Adaptive replanning.}
\label{app:replan_setting}
{Full rollout evaluation uses the same adaptive replanning rule of CD. \edited{Let $\tau$ be the tracking threshold in Table~\ref{tab:ogbench_config}, $M$ the maximum number of replans, $a_{\mathrm{back}}$ the waypoint-backoff parameter, and $c_{\mathrm{delay}}$ the additional condition delay.} The evaluator triggers a new plan when either (i) the distance between the current observation and the current planned waypoint exceeds $\tau$, or (ii) the agent has consumed sufficiently many waypoints from the current plan. In the latter case, the number of remaining chunks is reduced proportionally to the unused fraction of the previous plan. \edited{We set $M=10$ for environments that use adaptive replanning and use the environment-specific $(a_{\mathrm{back}},c_{\mathrm{delay}})$ constants in Table~\ref{tab:replan_constants}.}}

\begin{table*}[h]
\centering
\footnotesize
\makeblue
\setlength{\tabcolsep}{4pt}
\begin{tabular}{l c c c}
\toprule
\textbf{Environment group} & \textbf{Max replans} $M$ & \textbf{Backoff} $a$ & \textbf{Delay} $c$ \\
\midrule
\texttt{pointmaze-*} & 10 & 10 & 150 \\
\texttt{antmaze-*} and \texttt{antmaze-*-o15d} & 10 & 50 & 150 \\
\texttt{humanoidmaze-*} & 10 & 300 & 150 \\
\texttt{antsoccer-arena} & 10 & 50 & 50 \\
\texttt{antsoccer-medium} & 10 & 0 & 10 \\
\bottomrule
\end{tabular}
\vspace{1em}
\caption{\textbf{Adaptive replanning constants.} {The threshold $\tau$ and rollout horizon $N_{\max}$ are listed in Table~\ref{tab:ogbench_config}. PointMaze Medium/Large are evaluated without adaptive replanning in the CD-aligned setting because their default maximum number of replans is zero.}}
\label{tab:replan_constants}
\end{table*}
}

\subsection{ECD Settings and Hyperparameters}
\label{app:ecd_settings}

\edited{All main ECD runs use the CD-style interleaved chunk sampler and the bridge/boundary-reaction correction described in Algorithm~\ref{alg:impl_ecd_ddim}. The principal ECD hyperparameters are the bridge scale $\eta_b$, reaction scale $\eta_r$, delta clip $c$, Markov coupling $\rho$, and candidate ranker. Based on validation runs, we use boundary-mismatch ranking for the main reported experiments. Detailed settings can be found in our codebase.}

\subsection{Baselines}
\label{app:baselines}

\paragraph{Generative Skill Chaining~\citep{mishra2023generative} and CompDiffuser~\citep{luo2025compdiffuser}.}
{We follow the official CompDiffuser implementation and hyperparameter conventions: \url{https://github.com/devinluo27/comp_diffuser_release}. We use the same planner checkpoints for CD and ECD. PointMaze is executed with the PD controller. AntMaze, HumanoidMaze, AntSoccer, and AntMaze-o15d use inverse-dynamics models trained with the same architecture family described in Table~\ref{tab:invdyn_network}. For HumanoidMaze and AntSoccer, which are not fully covered by the released pretrained checkpoints, we train our own planner and inverse-dynamics models following the CD codebase.}

\paragraph{Compositional Diffusion with Guided Search~\citep{mishra2025compositional}.}
{We use the official CDGS codebase: \url{https://github.com/UtkarshMishra04/CDGS_ogbench}. We use the same checkpoints as CD and ECD. For CDGS, we report the resampling strategy with resampling rounds fixed to $4$ across mazes. We also report the CDGS with the $8$ resampling rounds in our repository which shows diminishing returns over $4$ rounds.
 Since this already substantially exceeds the local forward-evaluation budget of CD and ECD, we do not additionally scale guided-search compute further.}

\paragraph{Goal-conditioned offline RL algorithms.}
{For offline RL baselines, we follow the implementation setup established by OGBench}~\citep{ogbench_park2024}. {These baselines are not diffusion chunk-stitching methods and therefore do not share the Local-NFE accounting above; they are included to contextualize planner success rates in the standard OGBench benchmark.}

\section{Limitations and Future Work}
\label{app:diss}

{This appendix closes the paper by separating implementation assumptions from conceptual limitations. The main method replaces heuristic stitching with an energy-derived correction; the discussion below clarifies when this correction is reliable, where the current approximation may break down, and which extensions are most natural.}

\subsection{Limitations}

\paragraph{Score approximation depends on local Markov structure.}
{The efficient block-tridiagonal solve assumes a local first-order Markov surrogate over states (Eq.~\ref{eq:app_markov_surrogate}). This assumption is reasonable for many navigation tasks, but it can be weak for contact-rich manipulation, delayed dynamics, long-range constraints, or settings where feasibility depends on history. \edited{In such cases, the surrogate can underestimate cross-time coupling, weakening boundary messages and causing slower mode commitment or residual boundary inconsistency.}}

\paragraph{Sensitivity to chunking and overlap design.}
{As in prior compositional diffusion planners~\citep{luo2025compdiffuser, mishra2025compositional}, ECD still requires choosing the number of chunks $K$, chunk horizon $l$, and stride $d$. The overlap length is $o=l-d$. Very large overlaps ($d \ll l$) add redundancy and can make the effective energy landscape poorly conditioned, whereas very small overlaps ($o \ll l$) reduce the amount of information passed between neighboring chunks. We found stable settings across the OGBench tasks, but optimal values can remain domain- and task-dependent.}

\paragraph{Limited evaluation scope.}
{Our experiments focus on OGBench stitching tasks in navigation-like domains. These tasks capture long-horizon composition under multimodality and obstacles, but they do not cover the full range of issues that arise in real-robot deployment, including perception noise, contact dynamics, actuator limits, hardware safety constraints, and model mismatch between offline data and online execution.}

\subsection{Discussion and Future Work}

\paragraph{Integrability versus the true data distribution.}
{ECD guarantees that the composed update field is conservative for the energy model we define, but this should not be interpreted as recovering the exact score of the true long-horizon data distribution. Integrability is necessary for a true score field, not sufficient. Our bridge energy is built from squared deviation to the local denoiser's conditional mean, so the induced density is best viewed as a structured compatibility model. It enforces the consistency properties in \cref{sec:energy_properties}, rather than serving as an exact likelihood for long-horizon trajectories.}

{\edited{A natural extension is to replace the squared-error bridge with learned conditional energies that better approximate short-segment likelihoods given boundary conditions and noise level.} For example, a noise-conditional energy-based model could be trained to distinguish stitchable chunks from incompatible ones.}

\paragraph{Learning better reaction approximations.}
{The current approximation computes boundary messages from a hand-structured Markov surrogate. A direct extension is to learn the surrogate energy or to amortize the Jacobian-vector product that produces the message $\bm{m}_k(t)$, avoiding backpropagation through $f_\theta$ at inference time. Possible designs include neural potentials $E_\phi,E_\psi$ beyond local curvature estimation (Appendix~\ref{app:approx_close_exact}) or a message-passing network supervised by exact reaction terms on training chunks.}

\paragraph{Adaptive chunking for compositional planning.}
{Another direction is to choose $(K,l,d)$ automatically from task geometry, predicted uncertainty, or state-dependent difficulty. For example, the planner could increase overlap or reaction strength in ambiguous regions while keeping computation low in easier regions.}

\paragraph{Real-robot evaluation.}
{Finally, ECD should be evaluated under real-world disturbances such as actuation noise, partial observability, imperfect inverse dynamics, and dataset bias. Coupling the energy correction with online adaptation or uncertainty-aware replanning may improve reliability in deployment.}

\end{document}